\documentclass{article}
\usepackage[utf8]{inputenc}
\usepackage{graphicx}
\usepackage{fullpage}
\usepackage{amssymb}
\usepackage{float}
\usepackage{amsmath}
\usepackage{amsthm}
\usepackage{latexsym}
\usepackage{graphicx}
\usepackage[round]{natbib}
\usepackage{algorithm}
\usepackage{amsmath,amsfonts,enumitem,amsthm}
\usepackage{wrapfig}
\usepackage{color}
\usepackage{url}
\usepackage[noend]{algpseudocode}

\newcommand{\F}{{\mathcal F}}

\newcommand{\R}{{\mathbb R}}

\newcommand{\E}{{\mathbb{E}}}

\newtheorem{assumption}{\bf Assumption}

\newtheorem{lemma}{\bf Lemma}
\newtheorem{proposition}{\bf Proposition}

\newtheorem{theorem}{\bf Theorem}
\newtheorem{corollary}{\bf Corollary}

\newtheorem{remark}{\bf Remark}

\newcommand{\blue}[1]{{\color{black}#1}}
\newcommand{\newblue}[1]{{\color{black}#1}}
\newcommand{\bluest}[1]{{\color{black}#1}}

\title{On the Sample Complexity of Actor-Critic Method \\for Reinforcement Learning with Function Approximation }
\author{Harshat Kumar\footnote{Department of Electrical and Systems Engineering, University of Pennsylvania, Philadelphia, PA 19104}, Alec Koppel\footnote{JPMorgan AI Research, New York, NY}, and Alejandro Ribeiro\footnotemark[1] }

\begin{document}
\maketitle

\begin{abstract}
Reinforcement learning, mathematically described by Markov Decision Problems, may be approached either through dynamic programming or policy search. Actor-critic algorithms combine the merits of both approaches by alternating between steps to estimate the value function and policy gradient updates.  Due to the fact that the updates exhibit correlated noise and biased gradient updates, only the asymptotic behavior of actor-critic is known by connecting its behavior to dynamical systems. This work puts forth a new variant of actor-critic that employs Monte Carlo rollouts during the policy search updates, which results in controllable bias that depends on the number of critic evaluations. As a result, we are able to provide for the first time the convergence rate of actor-critic algorithms when the policy search step employs policy gradient, agnostic to the choice of policy evaluation technique. In particular, we establish conditions under which the sample complexity is comparable to stochastic gradient method for non-convex problems or slower as a result of the critic estimation error, which is the main complexity bottleneck. These results hold in continuous state and action spaces with linear function approximation for the value function. We then specialize these conceptual results to the case where the critic is estimated by Temporal Difference, Gradient Temporal Difference, and Accelerated Gradient Temporal Difference. These learning rates are then corroborated on a navigation problem involving an obstacle and the pendulum problem which provide insight into the interplay between optimization and generalization in reinforcement learning.
\end{abstract}

\section{Introduction} \label{sec:intro}
Actor-critic refers to a family of two time-scale algorithms for reinforcement learning where one alternates between policy gradient updates (actor) and action-value function estimation in an online fashion (critic). These approaches form the bedrock of several practical advances in reinforcement learning, as in supply chain management \citep{giannoccaro2002inventory}, power systems \citep{jiang2014comparison}, robotic manipulation \citep{kober2012reinforcement}, and games of various kinds \citep{tesauro1995temporal,brockman2016openai,mnih2016asynchronous,silver2017mastering}. While their asymptotic stability has been known for decades \citep{konda1999actor,konda2000actor}, their sample complexity is relatively unexplored. In this work, we establish the statistical behavior of actor-critic algorithms for a number of canonical settings, which to our knowledge is the first time a comprehensive accounting has been conducted. 

We focus on reinforcement learning problems over possibly continuous state and action spaces, which are defined by a Markov Decision Process \citep{puterman2014markov}: each time, starting from one state, an agent selects an action, and then it transitions to a new state according to a distribution Markov in the current state and action. Then, the environment reveals a reward informing the quality of that decision. The goal of the agent is to select an action sequence which yields the largest expected accumulation of rewards, defined as the value \citep{bellman1954theory,bertsekas2005dynamic}. Actor-critic algorithms adapt the merits of reinforcement learning algorithms based on approximate dynamic programming with those based on policy search, the two dominant model-free approaches in the literature \citep{sutton2017reinforcement}. 

\blue{For finite spaces, one may obtain the globally optimal policy, and therefore it is possible quantify sample complexity in terms of the gap to the optimal value function (regret) as, e.g., a polynomial function of the cardinality of the state and action spaces -- see \cite{NIPS2018_7735} and references therein. This is possible because these quantities have finite cardinality; however, in continuous spaces, these analyses break down because policy parameterization is required, and the value function becomes non-convex with respect to the policy parameters (unless it is parameterized by a sufficiently high-dimensional neural model \citep{wang2019neural}). }

More specifically, in the \emph{actor step} of actor-critic, stochastic gradient steps with respect to the value function over a parameterized family of policies are conducted. Via the Policy Gradient Theorem \citep{sutton2000policy}, the gradient with respect to policy parameters (policy gradient) is the product of two factors: the score function and the $Q$ function. One may employ Monte Carlo rollouts to estimate $Q$-factors, which under careful choice of rollout horizon, can be shown to be unbiased \citep{santi2018stochastic}. As a result, linking policy gradient methods to more standard stochastic programming results for non-convex optimization, namely, sublinear $\mathcal{O}(k^{-1/2})$ rates to stationarity have recently been established \citep{Zhang_CDC}. Doing so, however, requires an inordinate amount of querying to the environment in order to generate trajectory data. In actor-critic, we replace Monte Carlo rollouts with online estimates for the action-value function.

More specifically, in actor-critic, the \emph{critic step} estimates the action-value ($Q$) function through stochastic approximation, i.e., temporal difference (TD) \citep{sutton1988learning}, approaches to solving Bellman's evaluation equation \citep{watkins1992q,tsitsiklis1994asynchronous}. Combining temporal difference iterations with nonlinear function parameterizations may cause instability, as shown by \cite{baird1995residual,tsitsiklis1997analysis}. This motivates the majority of TD algorithms to focus on the case where the $Q$ function is parameterized by a linear basis expansion over given universal features, which is common in practice \citep{sutton2017reinforcement}, and can be satisfied by radial basis function (RBF) networks or auto-encoders \cite{park1991universal}. We consider this setting of universal features given \emph{a priori}.

The asymptotic stability of linear TD algorithms hinges upon dynamical systems tools to encapsulate the mean estimation error sequence -- see \cite{borkar2000ode,kushner2003stochastic}. By contrast, a number of finite-time characterizations of various TD algorithms have appeared recently, i.e., those based on stochastic fixed point iterations and gradient-based approximations known as gradient temporal difference (GTD) \citep{sutton2009fast}. For TD algorithms, finite-time sublinear rates have been derived both in the case where samples (state-action-reward triples) are independent and identically distributed (i.i.d.) \citep{dalal2018finite,bhandari2018finite,lakshminarayanan2018linear} and when they exhibit Markovian dependence \citep{srikant2019finite}. Further, the convergence of GTD was established in \citep{koppel2017breaking,tolstaya2018nonparametric} by employing coupled supermartingales \citep{wang2017stochastic}, which permits us to derive the rates of \blue{convergence in expectation of} GTD as corollaries. As a result, we may explicitly derive the bias due to critic estimation error in terms of the number of critic steps. This is in contrast to the use of an unbiased estimate from a Monte Carlo rollout, as in pure policy gradient methods. \blue{We further note that contemporaneously of beginning this work, several analyses of GTD have been developed \citep{liu2015finite,dalal2018finite,DalalST20} that refine the rates employed in this analysis; however, these results focus on concentration bounds (``lock-in probability"), a weaker metric of stability than convergence in mean, i.e., convergence in Lebesgue integral implies convergence in measure. Since in this work we focus on the intuitive and broadly interpretable \emph{global convergence} to stationarity in terms of the expected gradient norm of the value function, we seek to employ policy evaluation rates that are compatible with this goal, and defer refined lock-in probability results, for which tighter bounds of convergence on the critic exist, to future work.}
\vspace{3mm}

{\bf \noindent Convergence of Actor-Critic}  In this work, we link the behavior of actor-critic to gradient ascent algorithms with biased gradient directions. This bias is controllable and depends on the step-size and number of critic iterations per actor update. We perform this analysis for the setting that samples are i.i.d, which may be explicitly guaranteed through the introduction of a new Monte Carlo rollout step for each actor update. As a result, we  establish that actor-critic, independent of any critic method, exhibits convergence to stationary points of the value function that are comparable to stochastic gradient ascent in the non-convex regime. A key distinguishing feature from standard non-convex stochastic programming is that the rates are inherently tied to the bias of the search direction which is determined by the choice of critic scheme. In fact, our methodology is such that a rate for actor-critic can be derived for any critic-only method for which a convergence rate in expectation on the parameters can be expressed. In particular, we establish the rates for actor-critic with temporal difference (TD) \citep{sutton1988learning} and gradient TD (GTD) \citep{sutton2009fast} critic steps. Furthermore, we propose an Accelerated GTD (A-GTD) method derived from accelerations of stochastic compositional gradient descent \citep{wang2017stochastic}, which converges faster than TD and GTD.

\begin{table}[t]\label{tab:rates}
\begin{center}
\begin{tabular}{c|c|c|c|c}
 Critic Method& Convergence Rate & State-Action Space  & Smoothness Assumptions & Algorithm \\
\hline
GTD (SCGD) & $O\left(\epsilon^{-3}\right)$ & Continuous & Assumption \ref{as:SCGD} & Alg \ref{alg:SCGD} \\
\hline
GTD (A-SCGD) & $O\left(\epsilon^{-5/2}\right)$  &Continuous  &  Assumptions \ref{as:SCGD} and \ref{as:A-SCGD} & Alg \ref{alg:A-SCGD}\\
\hline
TD(0) & \blue{ $O\left(\epsilon^{-2 / \sigma }\right)$}  & Continuous & None & Alg \ref{alg:TD0}\\
\hline
TD(0) &   $O\left(\epsilon^{-2}\right)$ & Finite & None& Alg \ref{alg:TD0}
\end{tabular}
\end{center}
\caption{Rates of Actor Critic with Policy Gradient Actor updates and different critic-only methods.\blue{ The term $\sigma$ is the critic stepsize for TD(0) with continuous state-action space, and should be chosen according to conditioning of the feature space (see Section \ref{sec:td_continuous}).} }
\end{table}

\blue{In summary, for the continuous spaces, we establish that A-GTD converges faster than GTD, and the effective convergence rate of TD(0) varies as a result of the feature space representation selected \emph{a priori}. }
In particular, this introduces a trade off between the smoothness assumptions and the rates derived (see Table \ref{tab:rates}). TD has no additional smoothness assumptions, and it achieves a rate of \blue{$O(\epsilon^{-2/\sigma})$.} \blue{This rate is analogous to the non-convex analysis of stochastic compositional gradient descent when $\sigma$ is equal to $0.5$, which is a conservative estimate (see Figure \ref{fig:sigma_lambda}).} Adding a smoothness assumption, GTD achieves the faster rate of $O(\epsilon^{-3})$. By requiring an additional smoothness assumption, we find that A-GTD achieves the fastest convergence rate of $O(\epsilon^{-5/2})$. For the case of finite state action space, actor critic achieves a convergence rate of $O(\epsilon^{-2})$. Overall, the contribution in terms of sample complexities of different actor-critic algorithms may be found in Table \ref{tab:rates}. 

Relative to existing convergence results, actor-critic is classically studied as a form of two time-scale algorithm \citep{borkar1997stochastic}, whose asymptotic stability is well-known via dynamical systems \citep{kushner2003stochastic,borkar2009stochastic}. To wield these approaches to establish finite-time performance, however, concentration probabilities and geometric ergodicity assumptions of the Markov dynamics are required -- see \cite{borkar2009stochastic}. We obviate these complications by focusing on the case where independent trajectory samples are acquirable through querying the environment, for which recent unbiased sampling procedures gave proved adept \citep{santi2018stochastic,Zhang_CDC}.
Relative to existing finite-time characterizations of actor-critic, \citep{cai2019neural} proposes Neural TD updates, which converges to global optimality under a suitably over-parameterized deep neural network (DNN) and initialization. One quandary is how to find these initializations or design DNN architectures to satisfy these conditions. In separate work, the sample complexity of actor-critic has been established in terms of the value function gradient norm when the critic parameters are estimated with non-linear function approximation in a \emph{batch} fashion \citep{yang2018finite}.  \blue{ It is well-known that non-linear function approximators may diverge given by various counterexamples \citep{baird1995residual,tsitsiklis1997analysis}. Our work circumvents this obstacle by considering only well-behaved and well-studied linear function approximation, which includes commonly chosen radial basis function (RBF) networks and auto-encoders fixed at the outset of RL training.  }

Since the original date of submission, efforts to refine the analysis in this work exist: for instance, relaxations of assumptions on the sampling distribution to allow Markovian dependence \citep{shuang2019sample,xu2020non,wu2020finite} and augmentations of the critic objective for practical variance reduction \citep{parisi2019td}.  However, these works require the Markov transition density to mix at an exponentially fast rate in order to establish convergence. Thus, while i.i.d. sampling may be difficult to justify,  exponentially fast mixing often does not hold either, unless algorithm step-sizes are sent to null at an exponential rate. These intricacies have motivated experimental techniques to mitigate correlation among samples using replay buffers \citep{wang2016sample} and parallelization of queries to a generative model \citep{gruslys2018reactor}. However, their exact relationship to mixing rates is opaque. Therefore, for simplicity, in this work we focus on the  i.i.d. case. 

Moreover, sharper sample complexities for actor-critic have been developed \cite{shuang2019sample,xu2020non,wu2020finite}; however, they do not address the possibility of designing alternate policy evaluation schemes than TD(0) updates, and instead focus only on actor-critic in its vanilla form. This is because their perspective is on understanding the sample complexity of actor-critic alone, whereas we provide a unified perspective upon the basis of biased stochastic gradient iteration. In doing so, we are able incorporate a variety of critic updates and illuminate the interplay of problem smoothness, cardinality, and the choice of critic parameterization. In particular, the sample complexity of actor-critic with TD(0) updates for the tabular case given in Corollary \ref{corr:actor_critic_complexity_finite_td} matches \cite{xu2020non,wu2020finite}, but in continuous spaces, depending on the conditioning of the feature map covariance and other problem smoothness conditions, GTD or A-GTD may yield faster convergence, a facet elsewhere unaddressed in the literature.

\newblue{Even more recently, efforts have been made to improve upon the rate of convergence by considering regularized MDP's with overparametrized networks \citep{cayci2022finite}, single critic step \citep{olshevsky2022small}, and single trajectory actor updates \citep{chen2022finite}. Decentralized convergence rates have also been established \citep{chen2022sample, zeng2022learning}. \cite{shen2020asynchronous} show that for both i.i.d. and markovian sampling, there is a linear speedup for the decentralized setting whose is bottleneck is the slowest mixing chain. All of the aforementioned results require the assumption that the probability for any action given a state is strictly positive, which we do not require.
}


We evaluate actor-critic with TD, GTD, and A-GTD critic updates on both a navigation problem and the canonical pendulum problem. For the navigation problem, we find that indeed A-GTD converges faster than both GTD and TD. Interestingly, the stationary point it reaches is worse than GTD or TD. This suggests that the choice of critic scheme illuminates an interplay between optimization and generalization that is less-well understood in reinforcement learning \citep{boyan1995generalization,bousquet2002stability}.
\newblue{For the pendulum problem, we also find that A-GTD converges fastest with respect to the gradient norm, which is consistent with our main convergence results. In particular, we again find that the faster convergence in gradient norm results the stationary point having a lower cumulative reward. We additinally consider advantage actor-critic in our simulations \citep{mnih2016asynchronous}.  } A detailed discussion on the results and implications can be found in section \ref{sec:num_results}.
The remainder of the paper is organized as follows. Section \ref{sec:prob} describes the problem of reinforcement learning and defines common assumptions which we use in our analysis. In section \ref{sec:alg}, we derive a generic actor-critic algorithm from an optimization perspective and describe how the algorithm would be amended given different policy evaluation methods. The derivation of the convergence rate for generic actor-critic is presented in section \ref{sec:convergence_analysis}, and the specific analysis for Gradient, Accelerated Gradient, and vanilla Temporal Difference are characterized in sections \ref{sec:GTD} and \ref{sec:TD0}.

%

\section{Reinforcement Learning}\label{sec:prob}
We consider the Reinforcement Learning (RL) problem where an agent moves through a state space $\mathcal{S}$ and takes actions that belong to some action set $\mathcal{A}$, and the state/action spaces are assumed to be continuous compact subsets of Euclidean space: $\mathcal{S}\subset  \mathbb{R}^q$ and $\mathcal{A}\subset \mathbb{R}^p$.  Every time an action is taken, the agent transitions to its next state that depends only on its current state and action. Moreover, a reward is revealed by the environment. In this situation, the agent would like to accumulate as much reward as possible in the long term, which is referred to as value. Mathematically this problem definition may be encapsulated as a Markov decision process (MDP), which is a tuple $(\mathcal{S},\mathcal{A},\mathbb{P},R,\gamma)$ with Markov transition density $\mathbb{P}(s'\mid s,a):\mathcal{S}\times\mathcal{A}\to \mathbb{P}(\mathcal{S})$ that determines the probability of moving to state $s'$. Here, $\gamma\in(0,1)$ is the discount factor that parameterizes the value of a given sequence of actions, which we will define shortly.

At each time $t$, the agent executes an action $a_t\in\mathcal{A}$ given the current state $s_t\in\mathcal{S}$, following a stochastic policy $\pi:\mathcal{S}\to \mathbb{P}(\mathcal{A})$, i.e., $a_t\sim \pi(\cdot\mid s_t)$. 
Then, given the state-action pair $(s_t,a_t)$, the agent observes a (deterministic) reward $r_t=R(s_t,a_t)$ and transitions to a new state $s_t' \sim \mathbb{P}(\cdot \mid s_t, a_t)$ according to a Markov transition density. For any policy $\pi$, define the value function $V_{\pi}:\mathcal{S}\to\mathbb{R}$ as 
\begin{equation}\label{equ:value}
V_{\pi}(s):=\E_{a_t\sim \pi(\cdot\mid s_t),s_{t+1}\sim \mathbb{P}(\cdot\mid s_t,a_t)}\bigg(\sum_{t=0}^\infty \gamma^t r_t\mid s_0=s\bigg),
\end{equation}
which is a measure of the long term average reward accumulation discounted by $\gamma$. We can further define the value $V_{\pi}:\mathcal{S}\times\mathcal{A}\to\mathbb{R}$ conditioned on a given initial action as the action-value, or Q function as $Q_{\pi}(s,a)=\E\big(\sum_{t=0}^\infty \gamma^t r_t\mid s_0=s,a_0=a\big)$. Given any initial state $s_0$, the goal of the agent is to find the optimal policy $\pi$ that maximizes the long-term return $V_{\pi}(s_0)$, i.e., to solve the following optimization problem 
\begin{equation}\label{equ:max_goal}
\max_{\pi\in\Pi}~~J(\pi) \;, \  \text{where} \  \quad\ J(\pi):=V_{\pi}(s_0).
\end{equation}
In this work, we investigate actor-critic methods to solve \eqref{equ:max_goal}, which is a hybrid RL method that fuses key properties of policy search and approximate dynamic programming. To ground the discussion, we first derive the canonical policy search technique called policy gradient method, and explain how actor-critic augments policy gradient. Begin by noting that to address \eqref{equ:max_goal}, one must search over an arbitrarily complicated function class $\Pi$ which may include those which are unbounded and discontinuous. To mitigate this issue, we parameterize the policy $\pi$ by a vector $\theta\in\mathbb{R}^d$, i.e., $\pi=\pi_{\theta}$, yielding RL tools called \emph{policy gradient methods} \citep{konda2000actor,bhatnagar2009natural,castro2010convergent}. 
Under this specification, the search over arbitrarily complicated function class $\Pi$ to \eqref{equ:max_goal} may be reduced to Euclidean space $\mathbb{R}^d$, i.e., a vector-valued optimization, $\max_{\theta\in\mathbb{R}^d}J(\pi_\theta):=V_{\pi_{\theta}}(s_0)$. Subsequently, we denote $J(\pi_\theta)$ by $J(\theta)$ for notational convenience. 

We now make the following standard assumption on the regularity of the MDP problem and the parameterized policy $\pi_\theta$, which are the same conditions as \cite{Zhang_SICON}, as well as an assumption to bound the state-action feature representation.

\begin{assumption}\label{assum:regularity}
	Suppose the reward function $R$ and the parameterized policy $\pi_\theta$ satisfy the following conditions:
\begin{enumerate}[label=(\roman*)]
		\item The absolute value of the reward $R$ is   bounded   uniformly by $U_{R}$, i.e., $|R(s,a)|\in[0,U_{R}]$ for any $(s,a)\in\mathcal{S}\times\mathcal{A}$. \label{as:bounded_reward}
		\item The policy $\pi_{\theta}$ is  differentiable with respect to $\theta$, and the score function $\nabla\log\pi_{\theta}(a\mid s)$ is $L_\Theta$-Lipschitz and has bounded norm, i.e., for any  $(s,a)\in\mathcal{S}\times\mathcal{A}$, 
		\begin{align}
		&\|\nabla \log\pi_{\theta^1}(a\mid s)-\nabla \log\pi_{\theta^2}(a\mid s)\|\leq L_\Theta\cdot \|\theta^1-\theta^2\|,\text{~~for any~~} \theta^1,\theta^2,\label{equ:assum_L_Lip}\\
		&\|\nabla\log\pi_{\theta}(a\mid s)\|\leq B_{\Theta},\text{~~for any~~} \theta.\label{equ:assum_score_bnded}
		\end{align}
		\label{as:bounded_policy}
	\end{enumerate}
\end{assumption}
Note that the  boundedness of the reward function in Assumption \ref{assum:regularity}\ref{as:bounded_reward}  is standard in policy search  algorithms \citep{bhatnagar2008incremental,bhatnagar2009natural,castro2010convergent,zhang2018fully}. Observe that with  $R$, we have the Q function is absolutely upper bounded by $U_{R}/(1-\gamma)$, since by definition 
\begin{align}\label{equ:Q_bndness}
|Q_{\pi_\theta}(s,a)|\leq 
\sum_{t=0}^\infty \gamma^t \cdot U_{R}=\frac{U_{R}}{1-\gamma}, ~~\text{for any}~~ (s,a)\in\mathcal{S}\times\mathcal{A}.
\end{align}
The same bound also applies for $V_{\pi_\theta}(s)$ for any $\pi_\theta$ and $s\in\mathcal{S}$ and thus the objective $J(\theta)$ which is defined as $V_{\pi_\theta}(s_0)$, satisfies,
\begin{align} \label{equ:J_bound}
|V_{\pi_\theta}(s)|\leq\frac{U_{R}}{1-\gamma},~~\text{for any $s\in\mathcal{S}$},~~\quad |J(\theta)|\leq \frac{U_{R}}{1-\gamma}.
\end{align}  
We note that the conditions  \eqref{equ:assum_L_Lip} and \eqref{equ:assum_score_bnded} have appeared in recent analyses of policy search \citep{castro2010convergent,pirotta2015policy,papini2018stochastic}, and are satisfied by canonical policy parameterizations such as Boltzmann policy \citep{konda1999actor} and Gaussian policy \citep{doya2000reinforcement}. For example, for Gaussian policy\footnote{We observe that in practice, the action space $\mathcal{A}$ is bounded, which requires a truncated Gaussian policy to be used over $\mathcal{A}$, as in \citep{papini2018stochastic}. } in continuous spaces, $\pi_\theta(\cdot\mid s)=\mathcal{N}(\phi(s)^\top\theta,\sigma^2)$, where $\mathcal{N}(\mu,\sigma^2)$ denotes the Gaussian distribution with mean $\mu$ and variance $\sigma^2$ \blue{ and $\phi(s)$ denotes some state feature representation.} Then the score function has the form $[a-\phi(s)^\top\theta]\phi(s)/\sigma^2$, which satisfies \eqref{equ:assum_L_Lip} and \eqref{equ:assum_score_bnded} if the feature vectors $\phi(s)$ have bounded norm, the parameter $\theta$ lies some bounded set, and the action $a\in\mathcal{A}$ is bounded.

Generally, the value function is nonconvex with respect to the parameter $\theta$, meaning that obtaining a globally optimal solution to \eqref{equ:max_goal} is out of reach unless the problem has  additional structured properties,  as in  phase retrieval  \citep{sun2016geometric},  matrix factorization \citep{li2016symmetry}, and tensor decomposition \citep{ge2015escaping}, among others. Thus, our goal is to design actor-critic algorithms to attain stationary points of the value function $J(\theta)$. Moreover, we characterize the sample complexity of actor-critic, a noticeable gap in the literature for an algorithmic tool decades old \citep{konda1999actor} at the heart of the recent innovations of artificial intelligence architectures \citep{silver2017mastering}.

\section{From Policy Gradient to Actor-Critic}\label{sec:alg}

In this section, we derive actor-critic method \citep{konda1999actor} from an optimization perspective: we view actor-critic as a way of doing stochastic gradient ascent with biased ascent directions, and the magnitude of this bias is determined by the number of critic evaluations done in the inner loop of the algorithm. The building block of actor-critic is called policy gradient method, a type of direct policy search, based on stochastic gradient ascent. Begin by noting that the gradient of the objective $J(\theta)$ with respect to policy parameters $\theta$, owing to the Policy Gradient Theorem \citep{sutton2000policy}, has the following form:

\begin{align}
\nabla J(\theta)&=\int_{s\in\mathcal{S},a\in\mathcal{A}}\sum_{t=0}^\infty\gamma^t \cdot p(s_t=s\mid s_0,\pi_\theta)\cdot\nabla \pi_{\theta}(a\mid s)\cdot Q_{\pi_\theta}(s,a)dsda\label{equ:policy_grad_1}\\
&=\frac{1}{1-\gamma}\int_{s\in\mathcal{S},a\in\mathcal{A}}(1-\gamma)\sum_{t=0}^\infty\gamma^t \cdot p(s_t=s\mid s_0,\pi_\theta)\cdot\nabla \pi_{\theta}(a\mid s)\cdot Q_{\pi_\theta}(s,a)dsda\notag\\
&=\frac{1}{1-\gamma}\int_{s\in\mathcal{S},a\in\mathcal{A}}\rho_{\pi_\theta}(s)\cdot\pi_{\theta}(a\mid s)\cdot\nabla \log[\pi_{\theta}(a\mid s)]\cdot Q_{\pi_\theta}(s,a)dsda\notag\\
&=\frac{1}{1-\gamma}\cdot\mathbb{E}_{(s,a)\sim \rho_{\theta}(\cdot,\cdot)}\big[\nabla\log\pi_{\theta}(a\mid s)\cdot Q_{\pi_\theta}(s,a)\big]. \label{equ:policy_grad_3}
\end{align}
\blue{This expression follows from rolling the sum forward, repeatedly applying Bellman's evaluation equation, and exploiting the Markov property of the transition kernel, together with multiplying and dividing by $\pi_\theta$ and rewriting the denominator in terms of the score function via the fact that $\nabla_x \log(x) = 1/x$, as in \citep{sutton2000policy,Zhang_CDC}.}
In the preceding expression, $p(s_t=s\mid s_0,\pi_\theta)$ denotes the probability of state $s_t$ equals $s$ given initial state $s_0$ and policy $\theta$, which is occasionally referred to as the occupancy measure, or the Markov chain transition density induced by policy $\pi$. Moreover, $\rho_{\pi_\theta}(s)=(1-\gamma)\sum_{t=0}^\infty\gamma^t p(s_t=s\mid s_0,\pi_\theta)$ is the ergodic distribution associated with the MDP for fixed policy, which is shown to be a valid distribution \citep{sutton2000policy}. For future reference, we define $\rho_{\theta}(s,a)=\rho_{\pi_\theta}(s)\cdot \pi_{\theta}(a\mid s)$.
The derivative of the logarithm of the policy $\nabla\log[\pi_{\theta}(\cdot\mid s)]$ is usually referred to as the \emph{score function} corresponding to the probability distribution $\pi_{\theta}(\cdot\mid s)$ for any $s\in\mathcal{S}$.

Next, we discuss how \eqref{equ:policy_grad_3} can be used to develop stochastic methods to address \eqref{equ:max_goal}. Unbiased samples of the gradient $\nabla J(\theta)$  are required to perform the stochastic gradient ascent, which hopefully converges to a stationary solution of the nonconvex maximization. One way to obtain an estimate of the gradient $\nabla J(\theta)$ is to evaluate the score function and $Q$ function at the end of a rollout whose length is drawn from a geometric distribution with parameter $1-\gamma$ \citep{Zhang_SICON}[Theorem 4.3]. If the $Q$ function evaluation is unbiased, then the stochastic estimate of the gradient $\nabla J(\theta)$ is unbiased as well. We therefore define the stochastic estimate by
\begin{equation} \label{equ:gradient_estimate}
\hat\nabla J(\theta) := \frac{1}{1 - \gamma} \hat Q_{\pi_\theta} (s_T, a_T) \nabla \log \pi_{\theta}(a_T\vert s_T),
\end{equation}
\blue{where the tuple $(s_T, a_T)$ is drawn from end of the geometric rollout of length $T\sim\textbf{Geom}(1-\gamma)$. Of course, such an approach is very inefficient with respect to samples, as it does not utilize the state action transitions up until the final tuple. Using the entire trajectory for the actor update comes at the cost of a biased gradient estimate. Before we characterize this bias, we will discuss how to evaluate the $Q$ function using the single point estimation for simplicity.}

We consider the case where the $Q$ function admits a linear parametrization of the form $\hat Q_{\pi_\theta}(s,a) = \xi^\top \varphi(s,a)$, which in the literature on policy search is referred to as the \emph{critic} \citep{konda1999actor}, as it ``criticizes" the performance of actions chosen according to policy $\pi$.  \blue{We let $\varphi:\mathcal{S} \times\mathcal{A} \rightarrow \mathbb{R}^p$ be a (possibly nonlinear) feature map such as a network of radial basis functions or an auto-encoder known \emph{a priori}.} \blue{ The choice to consider the $Q$ function with a linear function approximator comes from the well known convergence results of linear critic-only methods. In contrast, nonlinear function approximators suffer from the possibility of divergence, as is demonstrated by well known counterexamples \citep{baird1995residual,tsitsiklis1997analysis}.}

\blue{The critic parameter $\xi$ belongs to a bounded set $\xi\in \Xi \subset \mathbb{R}^p$ such that
\begin{equation}\label{equ_bounded_xi}
\|\xi\| \leq C_\xi \textrm{~for all~} \xi \in \Xi
\end{equation}} This is reasonable because \eqref{equ:Q_bndness} guarantees boundeness of the true $Q$ function. The boundedness of \blue{the estimate $\hat Q$} follows from requiring the feature map $\varphi(s,a)$ to be bounded, an assumption which can be achieved through normalization, \blue{which we subsequently state
\begin{assumption} \label{as:bounded_feature_map}
	For any state action pair $(s,a) \in \mathcal{S} \times \mathcal{A}$, the norm of the feature representation $\varphi(s,a)$ is bounded by a constant $C_\varphi\in \mathbb{R}_+$. 
\end{assumption}	
}
%
We also bound the true gradient of the objective function 
\begin{equation}\label{equ:grad_bound}
\| \nabla J(\theta_k)\| \leq C_\nabla,
\end{equation}
which is established by \eqref{equ:policy_grad_3} being bounded as a result of $|Q| \leq U_R/(1-\gamma)$ [c.f. \eqref{equ:Q_bndness}] and $\|\nabla \log \pi_\theta(a\vert s)\| \leq B_\Theta$ [c.f. \eqref{equ:assum_score_bnded}].

 Moreover, \blue{for each actor update $k$}, we estimate the parameter $\xi_k$ that defines the $Q$ function from \blue{an online} policy evaluation (critic-only) method after some $T_C(k)$ iterations, where $k$ denotes the number of policy gradient updates. Thus, we may write the stochastic gradient estimate as 
\begin{equation} \label{equ:estimate_xi}
\hat\nabla J(\theta) = \frac{1}{1 - \gamma} \xi_k^\top\varphi(s_T, a_T) \nabla \log \pi_{\theta}(a_T\vert s_T).
\end{equation}
If the estimate of the $Q$ function is unbiased, i.e., $\E[\xi_k^\top \varphi(s_T,a_T) \,|\,\theta, s, a]= Q(s,a)$, then $\E[\hat \nabla J(\theta) \,|\,\theta] = \nabla J(\theta)$ (c.f.  \citep{Zhang_SICON}[Theorem 4.3]). Typically, critic-only methods do not give unbiased estimates of the $Q$ function; however, in expectation the rate at which their bias decays is proportional to the number of $Q$ estimation steps. In particular, denote $\xi_*$ as the parameter for which the $Q$ estimate is unbiased:
\begin{equation} \label{equ:unbiased_stochastic}
\E[\xi_*^\top \varphi(s,a)] =  \E[\hat Q_{\pi_\theta}(s,a)] = Q(s,a).
\end{equation}
Hence, by adding and subtracting the true estimate of the parametrized $Q$ function to \eqref{equ:estimate_xi}, we arrive at the fact the policy search direction admits the following decomposition:
\begin{equation} \label{equ:estimate_star}
\hat\nabla J(\theta) = \frac{1}{1 - \gamma} (\xi_k - \xi_*)^\top\varphi(s_T, a_T) \nabla \log \pi_{\theta}(a_T\vert s_T) + \frac{1}{1 - \gamma} \xi_*^\top\varphi(s_T, a_T) \nabla \log \pi_{\theta}(a_T\vert s_T).
\end{equation}
The second term is the unbiased estimate of the gradient $\nabla J(\theta)$, whereas the first defines the difference of the critic parameter at iteration $k$ with the true estimate $\xi_*$. For linear parameterizations of the $Q$ function, policy evaluation methods establish convergence in mean of the bias
\begin{equation} \label{equ:critic_bound}
\E[\|\xi_k - \xi_*\|] \leq g(k),
\end{equation}
where $g(k)$ is some decreasing function.  We address cases where the critic bias decays at rate $k^{-b}$ for $b\in (0,1]$, due to the fact that several state of the art works on policy evaluation may be mapped to the form \eqref{equ:critic_bound} for this specification  \citep{wang2017stochastic,dalal2018finite}. \blue{We formalize this with the following proposition.
%
\begin{proposition} \label{prop:critic_bound}
\newblue{Given some $b \in (0,1]$, there exists a constant $L_1 > 0$ such that
\begin{equation} \label{equ:critic_bound2}
\E[\|\xi_k - \xi_*\|] \leq L_1k^{-b}.
\end{equation}
This implies the expected error of the critic parameter is bounded by $O(k^{-b})$.}
\end{proposition}
}%
Recently, alternate rates have been established as $O(\log k / k)$; however, they concede that $O(1/k)$ rates may be possible \citep{bhandari2018finite,zou2019finite}. Thus, we subsume recent sample complexity characterizations of policy evaluation as is described in Proposition \ref{prop:critic_bound}. \blue{Proposition \ref{prop:critic_bound} is an intrinsic property of many policy evaluation schemes, and thus permits one to substitute the standard subsampling rates of a Monte Carlo-based estimator for the Q function (as in REINFORCE \citep{sutton2000policy}) with one that is estimated online using, e.g., temporal difference learning. Hence its role is critical in relating the bias of using critic estimators rather than unbiased gradient estimates to the number of critic steps.
}

More specifically, \eqref{equ:estimate_star} is nearly a valid ascent direction: it is approximately an unbiased estimate of the gradient $\nabla J(\theta)$ since the first term becomes negligible as the number of critic estimation steps increases. Based upon this observation, we propose the following \blue{full trajectory} variant of actor-critic method \citep{konda1999actor}: run a critic estimator (policy evaluator) for $T_C(k)$ steps, whose output is critic parameters $\xi_{k}$. We denote the critic estimator by $\textbf{Critic:}\mathbb{N} \to  \mathbb{R}^p$ which returns the parameter $\xi_{k} \in \mathbb{R}^p$ after $T_C(k) \in \mathbb{N}$ iterations. 
Then, simulate a trajectory of length $H(k)$, and update the actor (policy) parameters $\theta$ as: \blue{
\begin{equation}\label{eq:actor_update}
\theta_{k+1} = \theta_k + \eta_k \frac{1}{1-\gamma} \sum_{t= 1}^{H(k)} \xi_{k}^\top \varphi(s_{t}, a_{t}) \nabla \log \pi_{\theta_k}(s_{t},a_{t}|\theta_k).
\end{equation}
Note that we make the number of critic estimation steps and horizon length grow with $k$. Increasing $T$ and $H$ with $k$ allows us to control the bias of the estimate as is seen in Proposition \ref{prop:critic_bound} for the critic evaluations and in the following theorem for horizon length.}

\blue{ Now, we will characterize the bias between the gradient estimate using the entire trajectory of length $H(k)$.  Let $\tau = \left\{s_1, a_1, \dots, s_{H-1}, a_{H-1}, s_{H}\right\}$ be a sampled trajectory of length $H$. Define $F_t$ to be the product of the true state action ($Q$) function with the score function evaluated at the tuple $(s_t, a_t)$, namely 
	\begin{equation}
	F_t := Q(s_t, a_t) \nabla_\theta \newblue{\log} \pi_\theta(s_t,a_t).
	\end{equation}
	One can consider constructing an estimate of the policy gradient using the entire trajectory of length $H$ by 
	\begin{equation}
	\hat g_H = \sum_{t = 1}^{H} \gamma^{t-1}F_t.
	\end{equation}
	The following theorem establishes the bias between the true policy gradient and the finite horizon estimate. 
	\begin{theorem} \label{thm:finite_horizon}
		Let Assumption \ref{assum:regularity} be in effect. Then it is true that for some finite $C_1$,
		$$\left\| \mathbb{E}_\tau\left[\hat g_H \right]- \nabla_\theta J(\theta) \right\| \leq \gamma^{H-1}C_1.$$
	\end{theorem}
	\begin{proof}
		First we will show that $\E_\tau \left[\sum_{t=1}^\infty \gamma^{t-1}F_t\right] = \nabla_\theta J(\theta)$. We let $\textrm{Pr}(s_t = s \vert s_1)$ denote the probability the state at time $t$ is equal to $s$ given the initial state $s_1$.
		\begin{equation}
		\begin{split}
		\mathbb{E}\left[\sum_{t=1}^\infty \gamma^{t-1} F_t\right]&= \sum_{t=1}^\infty \gamma^{t-1} \int_\mathcal{S} \mathbb{E}\left[F_t | s_t = s\right] \textrm{Pr}\left(s_t = s\vert s_1\right)\textrm{d}s \\
		&= \sum_{t = 1}^\infty \gamma^{t-1} \int_{\mathcal{S}} \int_{\mathcal{A}} Q(s,a)\nabla_\theta\newblue{\log} \pi_\theta (s,a) \textrm{d}a \textrm{Pr}(s_t = s|s_1) \textrm{d}s \\
		&= \int_\mathcal{S} \int_\mathcal{A} Q(s,a) \nabla_\theta \newblue{\log}\pi_\theta (s,a) \textrm{d}a \sum_{t = 1}^\infty \gamma^{t-1}\textrm{Pr}(s_t = s|s_1)\textrm{d}s \\
		&= \int_\mathcal{S} \int_\mathcal{A} Q(s,a) \nabla_\theta\newblue{\log} \pi_\theta (s,a) \textrm{d}a  \rho^{\pi_\theta}(s)\textrm{d}s\\
		&= \mathbb{E}_{s\sim \rho^{\pi_\theta}(s)} \left[\int_\mathcal{A} Q(s,a) \nabla_\theta\newblue{\log} \pi_\theta (s,a) \textrm{d}a\right]\\
		&= \mathbb{E}_{s\sim \rho^{\pi_\theta}(s), a \sim \pi_\theta(s, \cdot)} \left[Q(s,a) \nabla_\theta \log \pi_\theta(s,a)\right]\\
		&= \nabla_\theta J(\theta)
		\end{split}
		\end{equation}
		By Fubini's Theorem, we are able to exchange the summation and integrals due to the regularity assumptions. Let $\hat g_\infty = \sum_{t = 1}^\infty \gamma^{t-1} F_t$. Then 
		\begin{equation}
		\hat g_\infty - \hat g_H = \gamma^{H-1} \sum_{t = 0}^\infty \gamma^t F_{t+H +1}
		\end{equation}
\newblue{By the regularity assumptions, we can bound $F_t$ by $U_RB_\Theta / (1-\gamma)$.  As such, we establish the bound $\sum_{t = 0}^\infty \gamma^t F_{t+H +1}  \leq \sum_{t=0}^\infty \gamma^{t} U_RB_\Theta / (1-\gamma) \leq U_RB_\Theta /(1-\gamma)^2=:C_1 \leq \infty$}		
		Taking the norm of the expectation completes the proof.
	\end{proof}
	
}
\blue{
Theorem \ref{thm:finite_horizon} holds under the assumption that the true $Q$ function is accessible. Of course, only a biased version of the critic is available through the uses of a critic, as described before. The algorithm we propose is the actor-critic variant of the finite horizon gradient estimate. The actor parameter update takes the following form:
\begin{equation} \label{equ:fin_H_update}
\theta_{k+1} = \theta_k + \eta_k \hat g_{H}^{AC} =  \theta_k + \frac{1}{1-\gamma}\eta_k \sum_{t = 1}^{H(k)}  \gamma^{t-1} \xi_{k}^\top \varphi(s_{t}, a_{t}) \nabla \log \pi_{\theta_k}(s_{t},a_{t}|\theta_k).
\end{equation}
The following theorem characterizes the bias of the stochastic gradient estimate. 
\begin{theorem} \label{thm:finite_bias}
	Let Assumptions \ref{assum:regularity} and \ref{as:bounded_feature_map} be in effect. Then, when proposition \ref{prop:critic_bound} is in effect, it is true that for a horizon of length $H$ and $T$ critic evaluations,
	$$\left\|\E_\tau \left[\hat g_H^{AC}\right] - \nabla_\theta J(\theta) \right\| \leq C_1 \gamma^{H} + C_2T^{-b}$$
\end{theorem}
\begin{proof}
	Let $F_{AC,t}:= \xi_k^\top \varphi(s_t,a_t) \nabla_\theta \log \pi_{\theta}(s_t,a_t)$. Then 
	\begin{equation}
	\begin{split}
	\E_\tau\left[\hat g_{\infty}^{AC}\right] &= \E_\tau \left[\sum_{t=1}^\infty \gamma^{t-1} F_{AC,t}\right] \\
	&= \E_\tau \left[\sum_{t=1}^\infty \gamma^{t-1} \left(F_t +F_{AC,t} -F_t\right)\right]\\
	&= \E_\tau \left[\sum_{t=1}^\infty \gamma^{t-1} F_t\right] +  \E_\tau \left[\sum_{t=1}^\infty \gamma^{t-1} \left(F_{AC,t} -F_t\right)\right]\\
	&= \nabla_\theta J(\theta) + \E_\tau \left[\sum_{t=1}^\infty \gamma^{t-1} \left(F_{AC,t} -F_t\right)\right]\\ 
	\end{split}
	\end{equation}
	The final term can be considered an error term. Consider the difference 
\begin{equation}
F_{AC,t} - F_t = \left(Q(s_t,a_t) - \xi_k^\top \varphi(s_t,a_t) \right) \nabla \log \pi_{\theta}(s_t,a_t).
\end{equation}
Let $Q(s_t,a_t) = \xi_*^\top \varphi(s_t,a_t)$. Then by assumptions \ref{assum:regularity} and \ref{as:bounded_feature_map} and proposition \ref{prop:critic_bound}, 
\begin{equation}
|F_{AC,t} - F_t| \leq T^{-b}L_1C_\varphi B_\Theta
\end{equation}
This implies 
\begin{equation}
\left\|\hat g_{\infty}^{AC} - \nabla_\theta J(\theta)\right\| \leq T^{-b}L_1C_\varphi B_\Theta \frac{1}{1-\gamma} = C_2 T^{-b}
\end{equation}
Following the same logic as Theorem \ref{thm:finite_horizon}, we can bound the difference between the finite horizon estimate and the infinite horizon actor-critic estimate by
\begin{equation}
\|\hat g_\infty^{AC} - \hat g_H^{AC}\| \leq C_1 \gamma^{H-1}.
\end{equation}
We evoke triangle inequality to complete the proof.
\begin{equation}
\|\hat g_\infty - \hat g_H^{AC}\| = \|\hat g_\infty - \hat g_\infty^{AC} + \hat g_\infty^{AC} - \hat g_H^{AC}\| \leq \|\hat g_\infty - \hat g_\infty^{AC}\| + \|\hat g_\infty^{AC} - \hat g_H^{AC}\| \leq C_1 \gamma^{H-1} + C_2 T^{-b}.
\end{equation}
This concludes the proof.
\end{proof}

The fact that the estimate $\hat g^{AC}_{H}$ is bounded comes from the fact that $\hat g^{AC}_\infty$ is bounded. We formalize this for use in the analysis
\begin{equation} \label{eqn:actor_critic_bounded_variance}
\E(\|\hat g^{AC}_H \| ) \leq \E(\|\hat g^{AC}_\infty \| ) \leq   \frac{C_\varphi C_\xi B_\Theta}{(1- \gamma) } =: \sigma,
\end{equation}
where $C_\varphi$, $C_\xi$ and $B_\Theta$ come from Assumption \ref{as:bounded_feature_map}, \eqref{equ_bounded_xi} and Assumption \ref{assum:regularity} \ref{as:bounded_policy} respectively.
}

\newblue{Theorem \ref{thm:finite_bias} establishes the bias on the stochastic gradient update. The bias can be decreased by increasing $T$, the number of critic update steps per each actor step, and  $H$, the horizon for the actor update.  In our main result, we will set both of these quantities to grow linearly with $k$, meaning that we decrease the bias with each actor update step  (see Theorem \ref{thm:general_rate}).  In our numerical results, we show that selecting a large enough constant $T$ and $H$ is sufficient(see Section \ref{sec:num_results}) .     }
%
\begin{algorithm}[t]
	\caption{Generic Actor-Critic}
	\begin{algorithmic}[1]
		\Require 
		\Statex $s_0\in \mathbb{R}^n$, $\theta_0$, $\xi_0$,  stepsize  $\{\eta_k\}$, Policy evaluation method \textbf{Critic}: $\mathbb{N} \to  \mathbb{R}^p$, $\gamma \in (0,1)$ 
		\For{$k = 1, \dots$}
		\State $\xi_{k} \leftarrow$ \textbf{Critic}($T_C(k)$) ~~~\newblue{[e.g. See Alg \ref{alg:SCGD}-\ref{alg:TD0}]}
		\State \blue{$\theta_{k+1} \leftarrow \theta_k + \frac{1}{1-\gamma}\eta_k \sum_{t = 1}^{H(k)} \xi_{k}^\top \varphi(s_{t}, a_{t}) \nabla \log \pi_{\theta_k}(s_{t},a_{t}|\theta_k) $}
		\EndFor
	\end{algorithmic}
	\label{alg:AC_generic}
\end{algorithm}

We summarize the aforementioned procedure, which is agnostic to particular choice of critic estimator, as Algorithm \ref{alg:AC_generic}. \blue{We acknowledge that the actor-critic algorithm proposed in Algorithm \ref{alg:AC_generic} differs from \cite{konda1999actor} in that rather than updating the actor and critic in tandem, the critic learns the state-action (Q) function from scratch at each update of the actor algorithm. The classical version of the algorithm can be recovered by setting $T_C(k) = 1$ and initializing the critic parameter to the previous step. Existing convergence proofs of this format are limited to asymptotic convergence only, where the critic steps at a faster learning rate than the actor. As such, this batch-type approach emulates this behavior, as the critic must learn something meaningful before the actor can update. As such, one might relate our work to \cite{yang2018finite}; however, unlike their work, we are not only able to prove convergence to a stationary point of the original objective by increasing the number of critic evaluations at each actor step rather than keeping it fixed, but also, we use the entire trajectory rather than a single state action pair sampled from the discounted state distribution.}

{\bf \noindent Examples of Critic Updates} We note that $\textbf{Critic:}\mathbb{N} \to  \mathbb{R}^p$ admits two canonical forms: temporal difference (TD) \citep{sutton1988learning} and gradient temporal difference (GTD)-based estimators \citep{sutton2008convergent}. The TD update for the critic is given as
\begin{equation}\label{eq:td_update}
\delta_{t} = r_{t} + \blue{\gamma \xi^\top_t\varphi(s_t', a_t')   - \xi^\top_t\varphi(s_t,a_t)} \; , \quad \xi_{t+1} = \xi_t + \alpha_t\delta_t\varphi(s_t,a_t)
\end{equation}
whereas for the GTD-based estimator for the critic, we consider the update
\begin{align}\label{eq:gtd_update}
\delta_{t} &= r_{t} +  \blue{\gamma\xi_t^\top \varphi(s'_{t}, a'_t) - \xi^\top_t\varphi(s_t,a_t)}  \; ,
 \quad z_{t + 1} = (1- \beta_t)z_t + \beta_t \delta_t , \nonumber \\
 \xi_{t+1} &= \xi_t - 2\alpha_t z_{t+1}[\gamma\varphi(s'_{t}, a'_{t}) - \varphi(s_t,a_t) ]
\end{align}
We further analyze a modification of GTD updates proposed by \citep{wang2017stochastic} that incorporates an extrapolation technique to reduce bias in the estimates and improve error dependency, which is distinct from accelerated stochastic approximation with Nesterov Smoothing \citep{nesterov1983method}. With $y_0 = 0$ and $z_t$ defined for $t = 1, \dots$, the accelerated GTD (A-GTD) update becomes
\begin{align}\label{eq:accelerated_gtd_update}
\xi_{t+1} &= \xi_t - 2\alpha_t (\gamma \varphi(s'_t,a'_t) - \varphi(s_t,a_t))y_t \\
z_{t+1} &= -\left(\frac{1}{\beta_t} - 1\right)\xi_t + \frac{1}{\beta_t}\xi_{t+1} \nonumber \\
y_{t+1}& = (1-\beta_t)y_t + \beta_t (r(s_t,a_t) + z_{t+1}^\top\left(\gamma\varphi(s'_t,a'_t) -\varphi(s_t,a_t) \right) \nonumber
\end{align}
Subsequently, we shift focus to characterizing the mean convergence of actor-critic method given any policy evaluation method satisfying \eqref{equ:critic_bound} in Section \ref{sec:convergence_analysis}. Then, we specialize the sample complexity of actor-critic to the cases associated with critic updates \eqref{eq:td_update} - \eqref{eq:accelerated_gtd_update}, which we respectively call Classic (Algorithm \ref{alg:TD0}), Gradient (Algorithm \ref{alg:SCGD}), and Accelerated Actor-Critic (Algorithm \ref{alg:A-SCGD}).

\blue{ 

\begin{remark} \label{rem:advantages}
	We wish to emphasize that a major advantage of this generic characterization of actor-critic admits the ability to interchange critic only methods to estimate the state-action (Q)  function. The merit is twofold, as it can extend to faster convergence rates and fewer assumptions. In particular, recent works have shown tighter sample complexity bounds for critic-only methods for convergence in probability\cite{}, which suggests that existing bounds on convergence in expectation are not necessarily tight. Furthermore, so long as the convergence of the critic takes the form of Proposition \ref{prop:critic_bound}, the i.i.d. assumption for the critic can be lifted. The general conditions for stability of trajectories with Markov dependence, i.e., negative Lyapunov exponents for mixing rates, may be found in \citep{meyn2012markov}. 
\end{remark}
}

%

\section{Convergence Rate of Generic Actor-Critic} \label{sec:convergence_analysis}
In this section, we derive the rate of convergence in expectation for the variant of actor-critic defined in Algorithm \ref{alg:AC_generic}, which is agnostic to the particular choice of policy evaluation method used to estimate the $Q$ function used in the actor update. Unsurprisingly, we establish that the rate of convergence in expectation for actor-critic depends on the critic update used. Therefore, we present the main result in this paper for any generic critic method. Thereafter, we specialize this result to two well-known choices of policy evaluation previously described \eqref{eq:td_update} - \eqref{eq:gtd_update}, as well as a new variant that employs acceleration \eqref{eq:accelerated_gtd_update}.

We begin by noting that under Assumption \ref{assum:regularity}, one may establish Lipschitz continuity of the policy gradient $\nabla J(\theta)$  \citep{Zhang_SICON}[Lemma 4.2]. 
\begin{lemma}[Lipschitz-Continuity of Policy Gradient]\label{lemma:lip_policy_grad}
The policy gradient $\nabla J(\theta)$ is Lipschitz continuous with some constant $L>0$, i.e., 
	for any $\theta^1,\theta^2\in\mathbb{R}^d$
	\begin{align}
	\|\nabla J(\theta^1)-\nabla J(\theta^2)\|\leq L\cdot \|\theta^1-\theta^2\|. 
	\end{align}
\end{lemma}
\bluest{This lemma allows us to establish an approximate ascent for the objective sequence $\{J(\theta_k)\}$.}
%
%
%
\begin{lemma}  \label{lem:submart_prop}
Consider the actor parameter sequence defined by Algorithm \ref{alg:AC_generic}. Further let Assumptions \ref{assum:regularity} and \ref{as:bounded_feature_map} be in effect. \newblue{Define the probability space $\left(\Omega,\mathcal{F}, P \right)$. Further, let $\mathcal{F}_k$ be the $\sigma$-algebra generated by the set $\{s_u,a_u,\theta_u\}_{u< k} $, that is the states, actions, and policy parameters until time $k$. Then}, the sequence \bluest{$\{J(\theta_k)\}$} satisfies the inequality
\begin{equation} \label{equ:lem1_result} 
\mathbb{E}[\bluest{J(\theta_{k+1})} \mid \mathcal{F}_k] \geq \bluest{J(\theta_k)}  + \eta_k \|\nabla J(\theta_k)\|^2 - \eta_k C_\nabla C_1 \gamma^{H(k)-1} - \eta_k C_\nabla C_2 T(k)^{-b}- \bluest{L \sigma^2 \eta_k^2}.
\end{equation}
where $C_1$ and $C_2$ come from Theorem \ref{thm:finite_bias}.
\end{lemma}
\begin{proof} See Appendix \ref{prof:lem2}
\end{proof}
From \eqref{equ:lem1_result} (Lemma \ref{lem:submart_prop}), consider taking the total expectation 
\begin{equation}
\mathbb{E}[J(\theta_{k+1}) ] \geq \mathbb{E}[J(\theta_k)]  + \eta_k \mathbb{E}[ \|\nabla J(\theta_k)\|^2] - \eta_k C_\nabla C_1 \gamma^{H(k)-1} - \eta_k C_\nabla C_2 T(k)^{-b}\bluest{- L \sigma^2 \eta_k^2}.
\end{equation}
This almost describes an ascent of \bluest{$J(\theta_k)$}. Because the norm of the gradient is non-negative, if the latter three terms were removed, an argument could be constructed to show that in expectation, the gradient converges to zero. Unfortunately, both the error of the finite horizon estimate and the critic error complicate the picture. However, we know that the error goes to zero in expectation as the number of critic steps and the horizon length increase. Thus, we leverage this property to derive the sample complexity of actor-critic (Algorithm \ref{alg:AC_generic}).

We now present our main result, which is the convergence rate of actor-critic method when the algorithm remains agnostic to the particular \blue{choice} of critic scheme. We characterize the rate of convergence by the smallest number $K_\epsilon$  of actor updates $k$ required to attain a value function gradient smaller \blue{than} $\epsilon$, i.e. for $\epsilon > 0$, 
\begin{equation} \label{equ:Keps_def}
K_\epsilon = \min\{k\,:\,\inf_{0\leq m \leq k} \|\nabla J(\theta_m)\|^2 < \epsilon\}.
\end{equation}
\begin{theorem} \label{thm:general_rate}
Suppose the \blue{actor} step-size satisfies $\eta_k = k^{-a}$ for $a >0$ and  the critic update satisfies \blue{Proposition} \ref{prop:critic_bound}. \blue{Further let $T_C(k) = k + 1\cdot \mathbf{1}(b = 1)$, and $H(k) = k$.}
%
Then the actor sequence defined by Algorithm \ref{alg:AC_generic} satisfies
\begin{equation}\label{eq:thm1_main}
%
K_\epsilon \leq \mathcal{O}\left( \epsilon^{-1/\ell} \right) \; , \text { where } \ell = \min\{a,1-a, b\}
\end{equation}
Minimizing over $a$ yields actor step-size $\eta_k = k^{-1/2}$. Moreover, depending on the rate $b$ of attenuation of the critic bias [cf. \eqref{equ:critic_bound}], the resulting sample complexity is:
\begin{align}\label{eq:thm1_main2}
K_\epsilon \leq
  \begin{cases} 
   \mathcal{O}\left( \epsilon^{-1/b}\right) & \text{if } b\in (0,1/2)\\
    \mathcal{O}\left( \epsilon^{-2}\right).       & \text{if }b\in (1/2, 1]
  \end{cases}
\end{align}
%
%
\end{theorem}
\begin{proof} See Appendix \ref{proof:thm1}
\end{proof}
The analysis of Lemma \ref{lem:submart_prop} and Theorem \ref{thm:general_rate} do not make any assumptions on the size of the state action space. Additionally, the result describes the number of actor updates required. The number of critic updates required is simply the $K_\epsilon^\textrm{th}$ triangular number, that is $K_\epsilon + 1 \choose 2$. These results connect actor-critic algorithms with the behavior of stochastic gradient method for finding the root of a non-convex objective. Under additional conditions, actor-critic with TD updates for the critic step attains a $O( \epsilon^{-2} ) $ rate. However, under milder conditions on the state and action spaces but more stringent smoothness conditions on the reward function, using GTD updates for the critic yields $O(\epsilon^{-3})$ rates. These results are formally derived in the following subsections.
 \blue{We further note that contemporaneously of beginning this work, several refined analyses of TD and GTD have been developed \citep{dalal2018finite,DalalST20} that focus on concentration bounds (``lock-in probability"), a weaker metric of stability than convergence in mean, i.e., convergence in Lebesgue integral implies convergence in measure. In this work, we focus on \emph{global convergence} to stationarity in terms of the expected gradient norm of the value function, and thus employ policy evaluation rates that are compatible with this goal, i.e., rates in the form of attenuation of mean square error. We defer the study of lock-in probabilities to future work.}

\blue{
\begin{remark}
We note that it may be possible to establish convergence in terms of asymptotic covariance or the Hessian around a stationary point, as in \citep{stsy20180019}, and thus obtain a sharper characterization of the limit points of actor-critic. However, doing so pre-supposes that the algorithm settle to a neighborhood of a local extrema, and would require a Hessian parameterization that is only \emph{locally} valid. Hence sharper global convergence characterizations, to our knowledge, are beyond reach. In this work, our intention is to establish the \emph{global} sample complexity of actor-critic type algorithms, and leave strengthening the local rates using, e.g., techniques developed in \citep{stsy20180019}, to future work.
\end{remark}
 }
 
%

\section{Rates of Gradient and Accelerated Actor-Critic} \label{sec:GTD}
In this section, we show how Algorithm \ref{alg:AC_generic} can be applied to derive the rate of actor-critic methods using Gradient Temporal Difference (GTD) as the critic update. Thus, we proceed with deriving GTD-style updates through links to compositional stochastic programming \citep{wang2017stochastic} which is also the perspective we adopted to derive rates in the previous section. For simplicity in notation, we let $Q$ stand for $Q_{\pi_\theta}$.  Begin by recalling that any critic method seeks a fixed point of the Bellman evaluation operator:
\begin{equation}
(T^{\pi_\theta}Q)(s,a) \triangleq r(s,a) + \gamma \E_{s'\in \mathcal{S},\blue{ a' \sim \pi_{\theta}(s')}}[Q(s', a')~|~s,a]
\end{equation} 
Since we focus on parameterizations of the $Q$ function by parameter vectors $\xi\in\mathbb{R}^d$ with some fixed feature map $\varphi$ which is learned \emph{a priori}, the Bellman operator simplifies 
\begin{equation}
 T^{\pi_\theta} Q_\xi (s,a) = \mathbb{E}_{s' \blue{\in \mathcal{S}},a'\sim\pi_\theta(s')} [r(s,a) +\gamma \xi^\top \varphi (s',a') | s,a ]
\label{newBel}
\end{equation}
The solution of the Bellman equation is its fixed point: $T^\pi Q(s,a) = Q(s,a)$ for all $s\in \mathcal{S}, a \in \mathcal{A}$. Thus, we seek $Q$ functions that minimize the (projected) Bellman error
\begin{equation} \label{equ:def_F}
\min_{\xi \in \Xi}  \| \Pi T^{\pi_\theta} Q_\xi - Q_\xi \|_\mu^2  =: F(\xi).
\end{equation}
where $\Xi \subseteq \mathbb{R}^p$ is a closed and convex feasible set. The Bellman error quantifies distance from the fixed point for a given $Q_\xi$. Here the projection and $\mu$-norm are respectively defined as 
\begin{equation}
    \Pi \hat Q = \arg \min_{ \mathrm{f} \in \mathcal{F} } \| \hat Q -\mathrm{f} \|_\mu \; , \qquad \|Q \|_\mu^2 = \int Q^2(s,a)\mu(\mathrm{d}s,\mathrm{d}a),
    \label{projection}
\end{equation}
%
%
%
This parameterization of $Q$ implies that we restrict the feasible set -- which is in general $B(\mathcal{S},\mathcal{A})$, the space of bounded continuous functions whose domain is $\mathcal{S}\times \mathcal{A}$ -- to be $\F = \{Q_\xi : \xi \in \Xi \subset \R^d\}$ (as in \citep{maei2010toward}). Without this parameterization, one would require searching over $B(\mathcal{S},\mathcal{A})$, whose complexity scales with the dimension of the state and action spaces \citep{bellman57a}, which is costly when dimensions are large, and downright impossible for continuous spaces \citep{powell2007approximate}.

Under certain mild conditions drawing tools from functional analysis, we can define a projection over a class of functions such that $\Pi\hat Q = \hat Q$. For example, Radial-Basis-Function (RBF) networks have been shown to be capable of approximating arbitrarily well functions in $L^p(\mathbb{R}^r)$ \citep[Theorem 1]{park1991universal}. Further, neural networks with one hidden layer and sigmoidal activation functions are known to approximate arbitrarily well continuous functions on the unit cube \citep[Theorem 1]{cybenko1989approximation}. 

By the definition of the $\mu$-norm, we can write $F$ [cf. \eqref{equ:def_F}] as an expectation
\begin{equation}
F(\xi) = \mathbb{E} [( T^{\pi_\theta} Q_\xi - Q_{\xi})^2] .
\label{error2}
\end{equation}
As such, we replace the Bellman operator in (\ref{error2}) with (\ref{newBel}) to obtain
\begin{equation}
F(\xi) = \mathbb{E}_{s,a\sim \pi_\theta(s)}\{ ( \mathbb{E}_{s',a'\sim \pi_\theta(s')} [r(s,a) +\gamma \xi^\top \varphi (s',a') | s,a\sim \pi_\theta(s) ] - \xi^\top \varphi(s,a))^2\}.
\label{Jcost}
\end{equation}
Pulling the last term into the inner expectation, $F(\xi)$ can be written as the function composition $F(\xi) = (f \circ g)(x) = f(g(x))$, where $f: \mathbb{R} \to \mathbb{R}$ and $g: \mathbb{R}^p \to \mathbb{R}$ take the form of expected values
\begin{equation} \label{equ:f_g_def_orig}
f(y) = \E_{(s,a)}[f_{(s,a)}(y)] \;,\qquad  g(\xi) = \E_{(s',a')}[g_{(s',a')}(\xi)~|~s,a\sim\pi_\theta(s)],
\end{equation}
where 
\begin{equation} \label{equ:f_g_def}
f_{(s,a)}(y) = y^2 \; , \qquad g_{(s',a')}(\xi) = r(s,a) + \gamma \xi^\top \varphi(s',a') - \xi^\top \varphi(s,a).
\end{equation}
%
%
Because $F(\xi)$ can be written as a nested expectations of convex functions, we can use Stochastic Compositional Gradient Descent (SCGD) for the critic update \citep{wang2017stochastic}. This requires the computation of the sample gradients for both $f$ and $g$ in \eqref{equ:f_g_def_orig}
\begin{equation} \label{equ:nabla_f_g}
\nabla f_{(s,a)}(y) =2 y \; , \qquad \nabla g_{(s',a')}(\xi) = \gamma  \varphi(s',a') -  \varphi(s,a) .
\end{equation}
The specification of SCGD to the Bellman evaluation error \eqref{Jcost} yields the GTD updates \eqref{eq:gtd_update} defined in Section \ref{sec:alg} -- see \citep{sutton2008convergent} for further details. 
We now turn to establishing the convergence rate in expectation for Algorithm \ref{alg:AC_generic} (substituting Algorithm \ref{alg:SCGD} for the $\textbf{Critic}(k)$) step using Theorem \ref{thm:general_rate}. Doing so requires the conditions of Theorem 3 from \cite{wang2017stochastic} to be satisfied, which we subsequently state. 
\begin{assumption} \label{as:SCGD}
~\\
\vspace{-4mm}
\begin{enumerate}[label=(\roman*)] 
\item The outer function $f$ is continuously differentiable, the inner function $g$ is continuous, the critic parameter feasible set $\Xi$ is closed and convex, and there exists at least one optimal solution to problem \eqref{equ:def_F}, namely $\xi^*\in \Xi$
\item The sample first order information is unbiased. That is, 
$$\E[g_{(s_0',a_0')}(\xi) ~|~s_0,a_0\sim\pi_\theta(s_0)] = g(\xi)$$
\item The function $\E[g(\xi)]$ [cf. \eqref{equ:f_g_def}] is $C_g$-Lipshitz continuous and the samples $g(\xi)$ and $\nabla g(\xi)$ have bounded second moments
$$\E[\|\nabla g_{(s_0',a_0')}(\xi)\|^2 ~|~s_0,a_0\sim\pi_\theta(s_0)] \leq C_g, \; \qquad \E[\|g_{(s_0',a_0')}(\xi) - g(\xi)\|^2] \leq V_g$$
\item The $f_{(s,a)}(y)$  has a Lipschitz continuous gradient such that 
$$\E[\|\nabla f_{(s_0,a_0)}(y)\|^2] \leq C_f \; \qquad \|\nabla f_{(s_0,a_0)}(y) - f_{(s_0,a_0)}(\bar y)\| \leq L_f\|y - \bar y\| $$
for all $y, \bar y \in \mathbb{R}$
\item The projected Bellman error is strongly convex with respect to the critic parameter $\xi$ in the sense that there exists a $\lambda$ such that
$$ \nabla^2 F(\xi)  \succeq \lambda I $$
\end{enumerate}
\end{assumption}

The first part of Assumption \ref{as:SCGD}(i) is trivially satisfied by the forms of $f$ and $g$  in \eqref{equ:f_g_def}. Assumption \ref{as:SCGD}(ii) requires that the state-action pairs used to update the critic parameter to be independently and identically distributed (i.i.d.), which is a common assumption unless one focuses on performance along a \emph{single trajectory}. Doing so requires tools from dynamical systems under appropriate mixing conditions on the Markov transition density \citep{borkar2009stochastic,antos2008fitted}, which we obviate here for simplicity and to clarify insights. We note that the sample complexity of policy evaluation along a trajectory has been established by \cite{bhandari2018finite}, but remains open for policy learning in continuous spaces. Moreover, i.i.d. sampling yields unbiasedness of certain gradient estimators and second-moment boundedness which are typical  for stochastic optimization \citep{bottou1998online}. We note that these conditions come directly from \cite{wang2017stochastic} -- here we translate them to the reinforcement learning context.

\begin{algorithm}[t]
\caption{Critic: Gradient Temporal Difference (GTD)}
\begin{algorithmic}[1]
\Require 
\Statex $s_0\in \mathbb{R}^n$, $\theta$, $\xi_0$, stepsizes $\{\alpha_t\} \subset \mathbb{R}^+, \{\beta_t\} \subset (0,1]$ which satisfy $\frac{\alpha_t -1}{\beta_t} \to 0$, Horizon $T_C$
\For{$t = 0, \dots, T_C - 1$}
\State Sample $s_t$ from the ergodic distribution and draw action $a_t \sim \pi^{\theta}$
\State Observe next state $s'_t \sim \textbf{P} (s_t, a_t, s'_t)$ and observe reward $r_t$
\State $\delta_{t} = r_{t} +  \xi_t^\top (\gamma\varphi(s'_{t}, a'_t) - \varphi(s_t,a_t))$
\State $z_{t + 1} = (1- \beta_t)z_t + \beta_t \delta_t$
\State  Update Critic: 
$$\xi_{t+1} = \xi_t - 2\alpha_t z_{t+1}[\gamma\varphi(s'_{t}, a'_{t}) - \varphi(s_t,a_t) ]$$
\EndFor
\end{algorithmic}
\label{alg:SCGD}
\end{algorithm}
We further require $F(\xi)$ to be strongly convex, so that \cite{wang2017stochastic}[Theorem 3 and Theorem 7] holds.
\blue{ Consider the Hessian 
	\begin{equation} \label{equ:Hessian}
	\nabla^2 F(\xi) = \mathbb{E}_{s,a} \left[\mathbb{E}_{s',a'} \left[\gamma \varphi(s',a') - \varphi(s,a)|s,a\right]^\top\mathbb{E}_{s',a'} \left[\gamma \varphi(s',a') - \varphi(s,a)|s,a\right]\right].
	\end{equation}
Due to its structure, and the i.i.d. assumption, the Hessian $\nabla^2 F(\xi)$ is known to be positive definite \cite{bertsekas1995dynamic, dalal2018finite}.
}
%
%
%
We can now combine the convergence result (Theorem 3) from \cite{wang2017stochastic} with Theorem \ref{thm:general_rate} to establish the rate of actor-critic with GTD updates for the critic, through connecting GTD and SCGD. We summarize the resulting method as Algorithm \ref{alg:SCGD}, which we call Gradient Actor-Critic.
%
\begin{algorithm}[t]
\caption{Critic: Accelerated Gradient Temporal Difference (AGTD)}
\begin{algorithmic}[1]
\Require 
\Statex $s_0\in \mathbb{R}^n$, $\theta$, $\xi_0$, stepsizes $\{\alpha_t\} \subset \mathbb{R}^+, \{\beta_t\} \subset (0,1]$ which satisfy $\frac{\alpha_t -1}{\beta_t} \to 0$, Horizion $T_C$
\State Initialize $y_0 \leftarrow 0$
\For{$t = 0, \dots, T_C(k) - 1$}
\State Sample $s_t$ from the ergodic distribution and draw action $a_t \sim \pi^{\theta}$
\State Observe next state $s'_t \sim \textbf{P} (s_t, a_t, s'_t)$ and observe reward $r_t$
\State  Update Critic: 
$$\xi_{t+1} = \xi_t - 2\alpha_t (\gamma \varphi(s',a') - \varphi(s,a))y_t$$
\State Update auxiliary Critic parameters $y_t$ and $z_t$
\begin{align*}
z_{t+1} &= -\left(\frac{1}{\beta_t} - 1\right)\xi_t + \frac{1}{\beta_t}\xi_{t+1} \\
y_{t+1}& = (1-\beta_t)y_t + \beta_t (r(s,a) + z_{t+1}^\top\left(\gamma\varphi(s',a') -\varphi(s,a) \right)
\end{align*}
\EndFor
\end{algorithmic}
\label{alg:A-SCGD}
\end{algorithm}
\begin{corollary} \label{cor:GTD}
Consider the actor parameter sequence defined by Algorithm \ref{alg:SCGD}. If the stepsize $\eta_k = k^{-1/2}$ and the critic stepsizes are $\alpha_t = 1/t\sigma$ and $\beta_t = 1/t^{2/3}$, then we have the following bound on $K_\epsilon$ defined in \eqref{equ:Keps_def}:
\begin{equation}
K_\epsilon \leq \mathcal{O}\left( \epsilon^{-3} \right). 
\end{equation}
\end{corollary}
\begin{proof} Here we invoke \citep[Theorem 3]{wang2017stochastic} which characterizes the rate of convergence for the critic parameter 
\begin{equation}
\E[\|\xi_k - \xi_*\|^2] \leq \mathcal{O}\left(k^{-2/3}\right).
\end{equation}
Applying Jensen's inequality, we have
\begin{equation}
\E[\|\xi_k - \xi_*\|]^2  \leq \E[\|\xi_k - \xi_*\|^2] \leq \mathcal{O}\left( k^{-2/3}\right),
\end{equation}
Taking the square root gives us 
\begin{equation}
\E[\|\xi_k - \xi_*\|]  \leq\mathcal{O}\left( k^{-1/3}\right).
\end{equation}
Therefore, \blue{$b = 1/3$} \blue{(c.f. Proposition \ref{prop:critic_bound})} in Theorem \ref{thm:general_rate}, which determines the $\mathcal{O}\left(\epsilon^{-3}\right)$ rate on $K_\epsilon$ in the preceding expression.
\end{proof}
%

Unsurprisingly, with additional smoothness assumptions, it is possible to obtain faster convergence through accelerated variants of GTD. The corresponding actor-critic method with Accelerated GTD updates is given by substituting Algorithm \ref{alg:A-SCGD} for $\textbf{Critic}(k)$ in Algorithm \ref{alg:AC_generic}, which we call Accelerated Actor-Critic.  The validity of accelerated rates, aside from Assumption \ref{as:SCGD}, requires imposing that the inner expectation has Lipschitz gradients and that sample gradients have boundedness properties which are formally stated below.
%
\begin{samepage}
\begin{assumption} \label{as:A-SCGD}
~\\
\begin{enumerate}[label=(\roman*)] 
\item There exists a constant scalar $L_g > 0$ such that
$$\|\nabla \E_{s',a'\sim \pi_\theta(s')}[g(\xi_1)] - \nabla \E_{s',a'\sim \pi_\theta(s')}[g(\xi_2)]\| \leq L_g \|\xi_1-\xi_2\|, \; \forall \xi_1,\xi_2\in \Xi$$
\item The sample gradients satisfy with probability 1 that
$$\E\left[ \|\nabla g(\xi)\|^4 ~|~s_0,a_0\right] \leq C^2_g, \; \forall \xi \in \Xi \; , \qquad
\E\left[ \|\nabla f(y)\|^4 \right] \leq C^2_f, \; \forall y \in \mathbb{R}^d$$
\end{enumerate}
\end{assumption}
\end{samepage}
With this additional smoothness assumption, sample complexity is reduced, as we state in the following corollary.
\begin{corollary}
\label{cor:GTD2}
Consider the actor parameter sequence defined by Algorithm \ref{alg:A-SCGD}. If the stepsize $\eta_k = k^{-1/2}$ and the critic stepsizes are $\alpha_t = 1/t\sigma$ and $\beta_t = 1/t^{4/5}$, then we have the following bound on $K_\epsilon$ defined in \eqref{equ:Keps_def}:
\begin{equation}
K_\epsilon \leq \mathcal{O}\left( \epsilon^{-5/2} \right). 
\end{equation}
\end{corollary}
\begin{proof} The proof is identical to the proof of Corollary \ref{cor:GTD} while invoking Theorem 7 from \cite{wang2017stochastic}.
\end{proof}
Corollary \ref{cor:GTD2} establishes a $\mathcal{O}(\epsilon^{-5/2})$ sample complexity of actor-critic when accelerated GTD steps are used for the critic update. This is the lowest complexity/fastest rate relative to all others analyzed in this work for continuous spaces. However, this fast rate requires the most stringent smoothness conditions. In the following section, we shift to the case where the critic is updated using vanilla TD(0) updates \eqref{eq:td_update}, which is the original form of actor-critic proposed by \cite{konda1999actor}.

\section{Sample Complexity of Classic Actor-Critic} \label{sec:TD0}
\begin{algorithm}[t]
\caption{Critic: Classical Temporal Difference (TD(0))}
\begin{algorithmic}[1]
\Require 
\Statex $s_0\in \mathbb{R}^n$, $\theta$, $\xi_0$, stepsizes $\{\alpha_t\}$, Horizon $T_C$
\State Initialize $z_0 \leftarrow 0$
\For{$t = 0, \dots, T_C - 1$}
\State Sample $s_t$ from the ergodic distribution and draw action $a_t \sim \pi^{\theta_k}$
\State Observe next state $s'_t \sim \textbf{P} (s_t, a_t, s'_t)$ and observe reward $r_t$
\State  $\delta_{t} = r_{t} +   \xi^\top_t  \left(\gamma\varphi(s_t', a_t') - \varphi(s_t,a_t)\right) $
\State  Update Critic (Q function estimate): 
$$\xi_{t+1} = \xi_t + \alpha_t\delta_t\varphi(s_t,a_t)$$
\EndFor
\end{algorithmic}
\label{alg:TD0}
\end{algorithm}

In this section, we derive convergence rates for actor-critic when the critic is updated using TD(0) as in \eqref{eq:td_update} for two different canonical settings: the case where the state space action is continuous (Sec. \ref{sec:td_continuous}) and when it is finite (Sec. \ref{sec:td_finite}). Both use TD(0) with linear function approximation in its unaltered form \citep{sutton1988learning}. We substitute Algorithm \ref{alg:TD0} for the $\textbf{Critic}(k)$ step in Algorithm \ref{alg:AC_generic}, which is the classical form of actor-critic given by \cite{konda1999actor,konda2000actor}, thus the name Classic Actor-Critic.

\subsection{Continuous State and Action Spaces}\label{sec:td_continuous}
The analysis for Continuous State Action space TD(0) with linear function approximation uses the analysis from \cite{dalal2017finite} to characterize the rate of convergence for the critic. Their analysis requires the following common assumption.
\begin{assumption} \label{assumption:bounded_second_moments} There exists a constant $K_s > 0$ such that for the filtration $\mathcal{G}_t$ defined for the TD(0) critic updates, we have
\begin{equation}
\E[\|M_{t+1}\|^2 \vert \mathcal{G}_t] \leq K_s[1 + \|\xi_t - \xi_*\|^2],
\end{equation}
where $M_{t+1}$ is defined as
\begin{equation}
M_{t+1} = \left(r_t  + \gamma \xi_t^\top \varphi(s_{t+1},a_{t+1}) - \xi_t^\top \varphi(s_{t},a_{t}) \right) \varphi(s_{t},a_{t}) - b + A
\end{equation}
where 
\begin{equation}
\begin{split}
b := \E_{s,a\sim \pi(s)}[r(s,a)\varphi(s,a)], & \textrm{~and~} A :=  \E_{s,a\sim \pi(s)}[\varphi(s,a)(\varphi(s,a) - \gamma \varphi(s',a'))^\top]\\
\end{split}
\end{equation}
\end{assumption}
Assumption \ref{assumption:bounded_second_moments} is known to hold when the samples have uniformly bounded second moments, which is a common assumption for convergence results \citep{sutton2009fast,sutton2009convergent}. \blue{In the same way the projected Bellman error is strongly convex [see \eqref{equ:Hessian}], it is known that $A$ is positive definite. As such, we define $\lambda_\textrm{TD} \in (0, \lambda_\textrm{min} (A + A^\top))$. The value of $\lambda_\textrm{TD}$ is conditioned on the feature representation of the state space, which is chosen \emph{a priori}. However, this value plays an important role in determining the rate of convergence for TD(0), as we see in the following corollary. }
\begin{corollary} \label{corr:cont}
Consider the actor parameter sequence defined by Algorithm \ref{alg:TD0}. Suppose the actor step-size is chosen as $\eta_k = k^{-1/2}$ and the critic step-size takes the form
$\alpha_t = {1}/{(t+1)^\sigma}$
%
where \blue{$\sigma \in (0,1)$}. Then, for large enough $k$,
\begin{equation}\label{eq:corr_con}
K_\epsilon \leq \mathcal{O}\left(\epsilon^{-2/\sigma}\right)
\end{equation}
\end{corollary}
\begin{proof}
Here we invoke the TD(0) convergence result from \cite[Theorem 3.1]{dalal2018finite} which establishes that
\begin{equation}\label{eq:td0_continuous_convergence}
\E[\|\xi_t - \xi_*\|^2] \leq K_1 e^{-\lambda_\textrm{TD} t^{1-\sigma}/2} + \frac{K_2}{t^\sigma}
\end{equation}
for some positive constants $K_1$ and $K_2$. \blue{For $\sigma$ not close to $1$, the first term is dominated by \newblue{ ${K_2}/{t^\sigma}$}}, which permits us to write that
\begin{equation}
\E[\|\xi_t - \xi_*\|^2] \leq \mathcal{O}\left(\frac{1}{t^\sigma}\right)
\end{equation}
Applying Jensen's inequality, we have 
\begin{equation}
\E[\|\xi_t - \xi_*\|]^2 \leq \E[\|\xi_t - \xi_*\|^2] \leq \mathcal{O}\left(\frac{1}{t^\sigma}\right).
\end{equation}
Taking the square root on both sides  gives us 
\begin{equation}
\E[\|\xi_t - \xi_*\|] \leq \mathcal{O}\left(\frac{1}{t^{\sigma/2}}\right),
\end{equation}
which means that the convergence rate statement of \blue{Proposition \ref{prop:critic_bound}} is satisfied with parameter $b = \sigma/2$. Because $\sigma < 1/2$, this specializes Theorem \ref{thm:general_rate}, specifically, \eqref{eq:thm1_main2} to case (i), which yields the rate
\begin{equation}
K_\epsilon \leq \mathcal{O}\left(\epsilon^{-2/\sigma} \right) \; .
\end{equation}
Thus the claim in Corollary \ref{corr:cont} is valid.
\end{proof}
\blue{The operative phrase in the proof of the previous theorem is \emph{for $\sigma$ not close to 1}. This is because we want the first second term of \eqref{eq:td0_continuous_convergence} to dominate the first term so that Proposition \ref{prop:critic_bound} holds. Asymptotically, this is not a problem, however for finite sample complexity, the point at which the exponential term is dominated by the second term is highly sensitive to both $\lambda_\textrm{TD}$ and $\sigma$. The choice of $\sigma$ can be chosen to be larger as the value of $\lambda_\textrm{TD}$ grows. The choice of $\sigma$ as a function of $\lambda_\textrm{TD}$ and the number of iterates is summarized in Figure \ref{fig:sigma_lambda}.}

\begin{figure}[!t]
	\centering
		\includegraphics[trim = {6cm, 8cm, 6cm, 8cm},width = .6\linewidth]{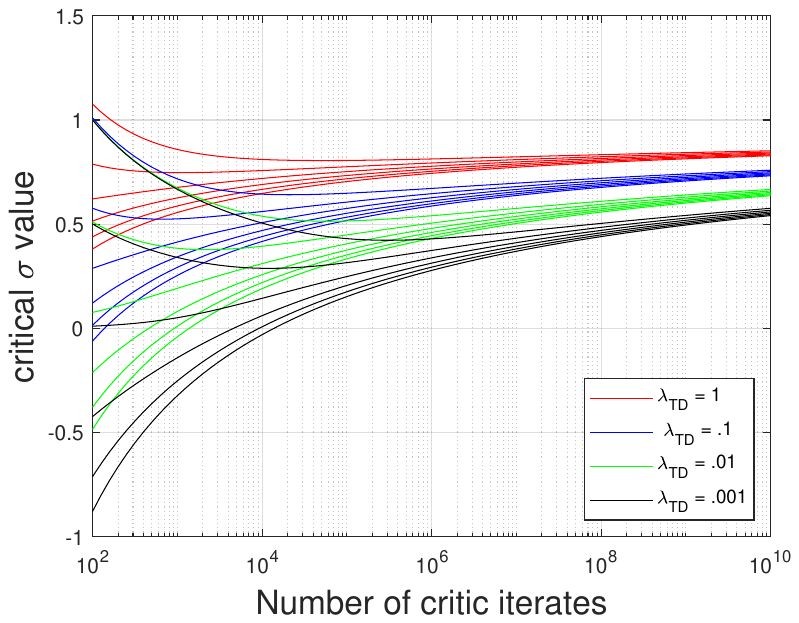}
		%
	\caption{\blue{ Plot shows the critical value of $\sigma$ for which the exponential term of \eqref{eq:actor_critic_complexity_finite_td} is dominated by the second term, thereby allowing Proposition \ref{prop:critic_bound} to hold. \emph{In particular, any $\sigma> 0$ chosen between zero and the curves shown above satisfies the proposition.} We show plots for varying values of $\lambda_\textrm{TD}$, which is determined by the feature space representation. For each value of $\lambda_\textrm{TD}$, we vary the ratio of the constants $K_2/K_1$ from $.001$ to $100$.  }}
	\label{fig:sigma_lambda}
\end{figure}

\blue{We find that as the value of $\lambda_\textrm{TD}$ increases, the critical value of $\sigma$ also increases. This means that the stepsize of the critic can be chosen to be larger, allowing for faster convergence.  Again, we define the critical value of $\sigma$ to be the point at which both terms on the right hand side of \eqref{eq:actor_critic_complexity_finite_td} are equal at a specific time $t$. Therefore, the feature space representation plays a large role on the performance of actor-critic with TD(0) updates. This result becomes apparent in our numerical results (section \ref{sec:num_results}). }
%
We note that, the GTD rates given in Corollary \ref{cor:GTD} hinge upon strong convexity of the projected Bellman error, which may hold for carefully chosen state-action feature maps, bounded parameter spaces, and lower bounds on the reward. These conditions are absent for TD(0) critic updates.

%

In the next section, we will consider analysis of actor-critic with TD(0) critic updates in the case where the state and action spaces are finite. As would be expected, this added assumption significantly improves the bound on the rate of convergence, i.e., reduces the sample complexity needed for policy parameters that are within $\epsilon$ of stationary points of the value function.

\subsection{Finite State and Action Spaces}\label{sec:td_finite}
In this section, we characterize the rate of convergence for the actor-critic defined by Algorithm \ref{alg:AC_generic} with TD(0) critic updates (Algorithm \ref{alg:TD0}) when the number of states and actions are finite, i.e., $|\mathcal{S}| = S<\infty$ and  $|\mathcal{A}| = A<\infty$. This setting yields faster convergence. 
A key quantity in the analysis of TD(0) in finite spaces is the minimal eigenvalue of the covariance of the feature map $\phi(s,a)$ weighted by policy $\pi(s)$, which is defined as 
\begin{equation}\label{eq:minimal_eigenvalue}
\omega =\min\left\{\textrm{eig} \left(\sum_{(s,a) \in \mathcal{S}\times \mathcal{A}} \pi(s) \varphi(s,a)\varphi(s,a)^\top\right)\right\} \; .
\end{equation}
 That $\omega$ exists is an artifact from the finite state action space assumption. \eqref{eq:minimal_eigenvalue} is used to define conditions on the rate of step-size attenuation for TD(0) [cf. \eqref{eq:td_update}] critic updates in \citep[Theorem 2 (c)]{bhandari2018finite}, which we invoke to establish the iteration complexity of actor-critic in finite spaces. We do so next.
%
\begin{corollary}\label{corr:actor_critic_complexity_finite_td} Consider the actor parameter sequence defined by Algorithm \ref{alg:TD0}. Let the actor step-size satisfy $\eta_k = k^{-1/2}$ and the critic step-size decrease as
%
$\alpha_t = {\beta}/({\lambda + t})$
%
where $\beta = 2/\omega(1-\gamma)$ and $\lambda = 16/\omega(1-\gamma)^2$. Then when the number of critic updates per actor update satisfies $T_C(k) = k+1$, the following convergence rate holds
\begin{equation}\label{eq:actor_critic_complexity_finite_td}
K_\epsilon \leq O\left(\epsilon^{-2}\right)
\end{equation}
\end{corollary}

\begin{proof}
We begin by invoking the TD(0) convergence result \citep[Theorem 2 (c)]{bhandari2018finite}:
\begin{equation}
\E[\|\xi_t - \xi_*\|^2] \leq \mathcal{O}\left(\frac{K_1}{t + K_2}\right),
\end{equation}
for some constants $K_1, K_2$ which depend on $\omega$ and $\sigma$. Applying Jensen's inequality, we have 
\begin{equation}
\E[\|\xi_t - \xi_*\|]^2 \leq \E[\|\xi_t - \xi_*\|^2] \leq \mathcal{O}\left(\frac{K_1}{t+K_2}\right).
\end{equation}
Taking the square root on both sides yields 
\begin{equation}
\E[\|\xi_t - \xi_*\|]\leq \mathcal{O}\left(\frac{K_1^{-1/2}}{(t+K_2)^{-1/2}}\right) \lesssim \mathcal{O}(t^{-1/2}),
\end{equation}
which means that Proposition \ref{prop:critic_bound} is valid with critic convergence rate parameter $b = 1/2$. Therefore, we may apply Theorem \ref{thm:general_rate} to obtain the rate 
\begin{equation}
K_\epsilon \leq \mathcal{O}\left(\epsilon^{-2} \right)
\end{equation}
as stated in Corollary \ref{corr:actor_critic_complexity_finite_td}.
\end{proof}


\section{Numerical Results} \label{sec:num_results}
In this section, we compare the convergence rates of actor-critic with the aforementioned critic-only methods on a two-dimensional navigation problem and the inverted pendulum. Before detailing the RL problem specifics, we first discuss the metrics we use to evaluate both performance and convergence. 

Because the main objective is to maximize the long term average reward accumulation, it follows naturally to measure the cumulative reward of a trajectory. We evaluate the policy without action noise $(\sigma^2 = 0)$, with a fixed trajectory length, and with a fixed starting position which makes the plots easier to compare. In addition, we consider a proxy for the gradient norm. \blue{In particular, we calculate the norm of the difference between two consecutive normalized actor parameters $(\|\theta_k/ \|\theta_k\| - \theta_{k+1}/\|\theta_{k+1}\|\|)$. The normalization treats two scaled versions of the same parameter equivalently. This is meaningful because the action vector field induced by the parameters (see Fig \ref{fig:gradients}) are similarly scaled versions of each other.} In this form, the gradient norm proxy serves as the optimization metric on which our main result is based. 

Along with varying the critic-only methods, we elect to consider two additional variations on policy gradient where the $Q$ function is replaced by the \emph{advantage} and \emph{value} functions. Recall the definition of the value function from \eqref{equ:value}. The advantage function is defined by $A(s_t, a_t) = Q(s_t, a_t) - V(s_t)$, which, by definition of the $Q$ function, can also take the form of $A(s_t, a_t) = r_{t+1} + \gamma V(s_{t+1}) - V(s_t)$ \citep{mnih2016asynchronous}. The main benefit of using the value function and advantage functions instead of the $Q$ function for actor critic is that the dimension of the function approximator domain is smaller, as the agent only needs to learn on the state space.  

\subsection{Navigating around an obstacle}
We consider the problem of a point agent starting at an initial state $s_0 \in \mathbb{R}^2$ whose objective is to navigate to a destination $s^*\in \mathbb{R}^2$ while remaining in the free space at all time. The free space \blue{$\mathcal{X}\subset \mathbb{R}^2$} is defined by 
\begin{equation}
\blue{\mathcal{X}} := \left\{s\in \mathbb{R}^2 \Big| \|s\|  \in [0.5,4] \right\}.
\end{equation}
The feature representation of the state is determined by a radial basis (Gaussian) kernel where 
\begin{equation} \label{equ:kernel}
\kappa(s,s') = \exp\left\{ \frac{-\|s-s'\|_2^2}{2\sigma^2}\right\}.
\end{equation}
The $p$ kernel points are chosen evenly on the $[-5,5]\times [-5,5]$ grid so that the the feature representation becomes 
\begin{equation}
\varphi(s) = \begin{bmatrix} \kappa(s,s_1) & \kappa (s,s_2)  & \dots & \kappa (s,s_p) 
\end{bmatrix}^\top,
\end{equation}
which we normalize. Given the state $s_t$, the action is sampled from a multivariate Gaussian distribution with covariance matrix $\Sigma = 0.5\cdot I_2$ and mean given by $\theta_k^\top \varphi(s_t)$. We let the action determine the direction in which the agent will move. As such, the state transition is determined by $s_{t+1} = s_t + 0.5a_t/\|a_t\|$. 

Because the agent's objective is to reach the target $s^*$ while remaining in $F$ for all time, we want to penalize the agent heavily for taking actions which result in the next step being outside the free space and reward the \blue{agent} for being close to the target. As such, we define the reward function to be 
\begin{equation} \label{equ:reward}
r_{t+1} = \begin{cases} -11 & \textrm{~if~} s_{t+1} \notin \blue{\mathcal{X}} \\ 
-0.1 & \textrm{~if~}  \|s_{t+1} - s^*\| < 0.5  \\
-1 & \textrm{~otherwise~}.
\end{cases}
\end{equation}
The design of this reward function for the navigation problem is informed by the \citep{Zhang_SICON}, which suggests that the reward function should be bounded away from zero. 
In this simulation, we allow for the agent to continue taking actions \emph{through} the obstacles. This formulation is similar to a car driving on a race track which has grass outside the track. The car is allow is allowed to drive off the track, however it incurs a larger cost due to the substandard driving conditions. 

Although it is true that this particular formulation does not allow for generalization, that is, if the target of the agent, obstacle location, or starting point of the agent are moved, the agent would have to start from scratch to learn a new meaningful policy, we emphasize that it is the rates of convergence which are of interest in this exposition, not necessarily finding the best way to design the navigation problem.\\
\blue{ \textbf{Algorithm Specifics:} We consider the problem with $\gamma = 0.97$. In practice, we use the entire trajectory data for the critic updates. In particular, } for each actor parameter update, we \blue{run} ten critic updates with rollout length $T = 66$ (comes from the expected rollout length given $\gamma = 0.97$). Similarly, we update the actor along the trajectory of rollout length \blue{H = 67}. 
For simulations, the actor update step $\eta_t$ is chosen to be constant $\eta =10^{-3}$. For TD(0), we let also let the critic stepsize be constant, namely  $\alpha_t = \alpha =  0.05$. For GTD, we let $\alpha_t = t^{-1}$ and $\beta_t = t^{-2/3}$. For A-GTD, we set $\alpha_t = t^{-1}$ and $\beta_t = t^{-4/5}$. We \blue{draw the initial distribution uniformly at random on the grid $[-2,2]\times [-2,2]$}, and we set the target to be $s^* = (-2,-2)$. For each critic only method, we run the algorithm \blue{50} times. We evaluate the policy by measuring the accumulated reward of a trajectory of length $H = 66$. 
\begin{figure}[!t]
	\centering
	\begin{tabular}{cc}
		\includegraphics[trim = {6cm, 8cm, 6cm, 8cm},width = .35\linewidth]{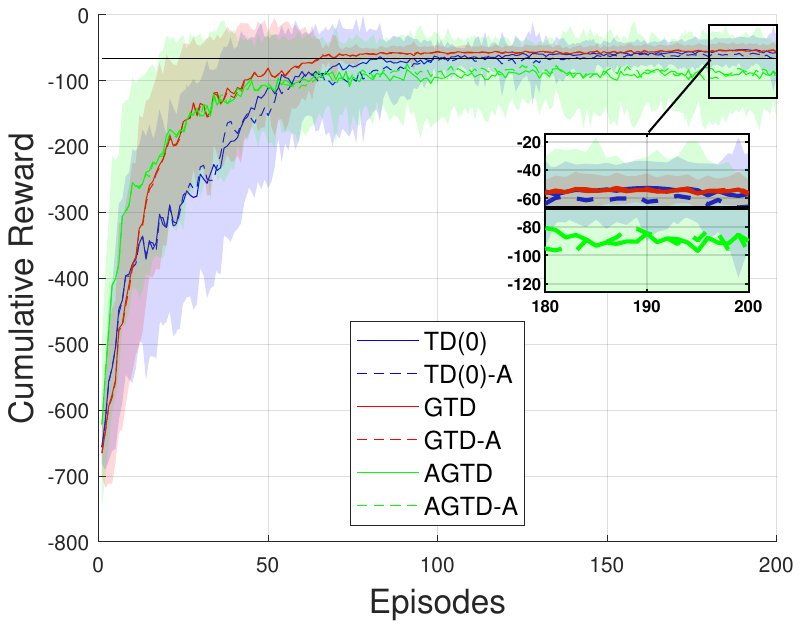} & \hspace{2cm}
		\includegraphics[trim = {6cm, 8cm, 6cm, 8cm},width = .35\linewidth]{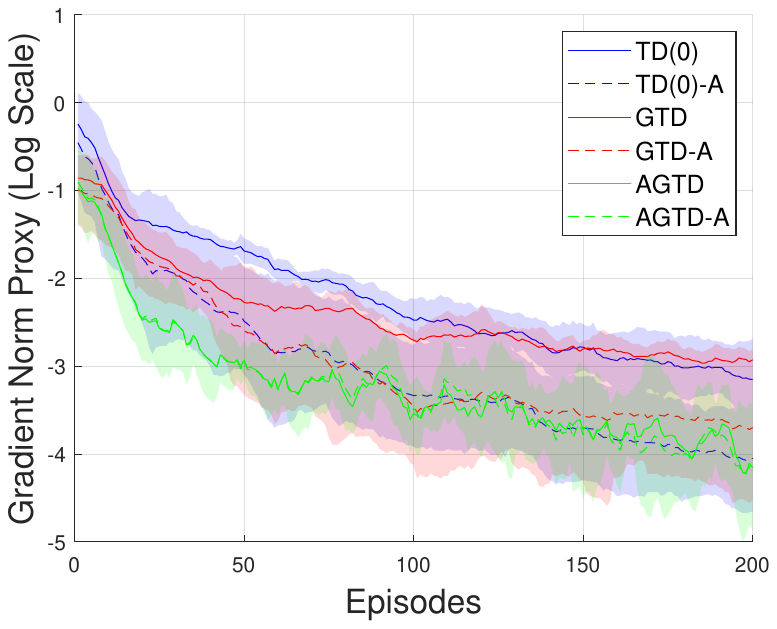} \\
		\small (a) & \hspace{2cm} \small (b)
	\end{tabular}
	\caption{Navigation Problem: (a) Average reward per episode with confidence bounds over 50 trials. (b) Average gradient norm proxy over 50 trials. A-GTD converges fastest with respect to the cumulative reward \emph{and} gradient norm proxy at the cost of converging to a suboptimal stationary point (see Fig. \ref{fig:gradients}). \newblue{A moving average filter of size ten has been applied on the gradient norm proxy to aid in comparison. }}
	\label{fig:rates}
\end{figure}

\begin{figure}[!t]
	\centering
	\begin{tabular}{ccc}
		\includegraphics[trim = {6cm, 8cm, 6cm, 8cm},width = .25\linewidth]{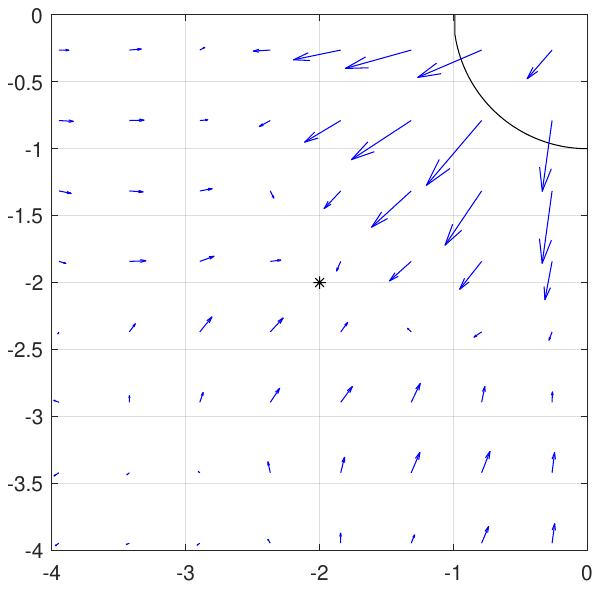} &\hspace{.9cm}
		\includegraphics[trim = {6cm, 8cm, 6cm, 8cm},width = .25\linewidth]{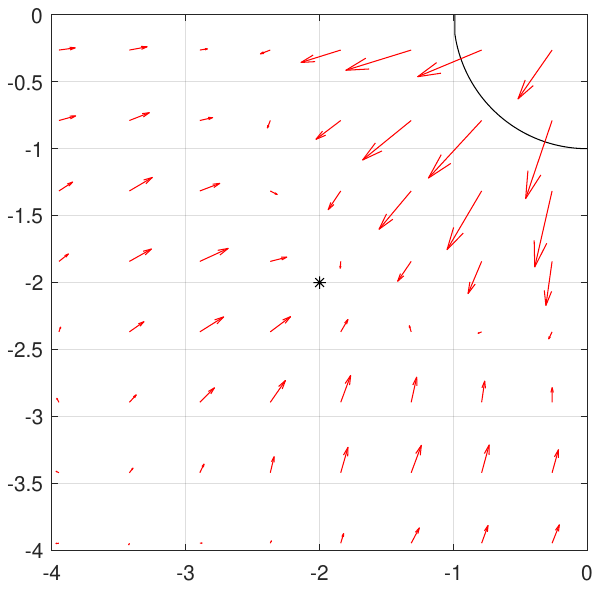} 
		&\hspace{.9cm}
		\includegraphics[trim = {6cm, 8cm, 6cm, 8cm},width = .25\linewidth]{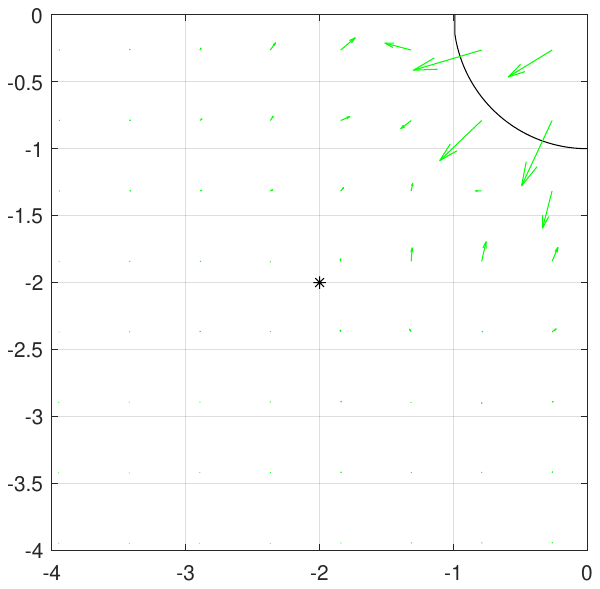} \\	
		\small (a) & \hspace{1cm} \small (b) & \hspace{1cm} \small (c) 
		
	\end{tabular}
	\caption{\blue{Visualization of the learned policy for the navigation problem. The obstacle is shown in the top right corner, and the target is located at (-2,-2). As Figure \ref{fig:rates}  (a) depicts, TD (shown in (a)) and GTD (shown in (b)) learn meaningful policies which guide the agent to the target. In contrast, A-GTD (shown in (c)) simply learns to avoid the obstacle.}}
	\label{fig:gradients}
\end{figure}

\subsection{Pendulum Problem}
\blue{
	We also consider the canonical continuous state action space reinforcement learning problem of the pendulum. The objective is to balance the pendulum upright starting from any starting position. Given that this is a well established benchmark for reinforcement learning, we refer the reader to \cite{brockman2016openai} for the specifications on reward and transition dynamics. Similar to the navigation problem, we let the feature representation of the state be determined by a radial basis (Gaussian) kernel (c.f. \eqref{equ:kernel}) where the $p$ kernel points are chosen evenly on $[-1, -1, -8, -2] \times [1, 1, 8, 2]$, where the bounds come from the sine and cosine of the angle $\theta$, the time derivative of the angle $\dot \theta$, and the maximum torque of the action respectively. The action is chosen by a normal distribution with mean $\xi^\top \varphi(s,a)$ and variance $\sigma_a^2$. \newblue{Like the navigation problem, we use a linear policy and linear critic. Again, we stress that these experiments are meant to show the rates of convergence, and not necessarily finding the best way to solve the pendulum problem.  For the pendulum problem, we only consider advantage actor-critic.}
	
	\blue{ \textbf{Algorithm Specifics:} Similar to the navigation problem, we let $\gamma = 0.97$, and we use the entire trajectory data for the critic updates. In particular, } for each actor parameter update, we \blue{run} ten critic updates with rollout length $T = 66$ (comes from the expected rollout length given $\gamma = 0.97$). Similarly, we update the actor along the trajectory of rollout length \blue{H = 66}. 
	For simulations, the actor update step $\eta_t$ is chosen to be constant $\eta =0.01$. \newblue{For critic only methods, we also let also let the critic stepsize be constant. In particular, we let $\alpha_t = 0.01$ for TD(0), $(\alpha_t, \beta_t) = (0.2, 0.01)$ for GTD, and $(\alpha_t, \beta_t) = (0.05, 0.005)$ for AGTD. We evaluate the policy by measuring the average accumulated by a single trajectory starting $\theta = \pi/2$ with angular velocity $\omega = 1$. The action variance is chosen to be $\sigma_a^2 = 0.5$}. 

}

\begin{figure}[!t]
	\centering
	\begin{tabular}{cc}
		\includegraphics[trim = {6cm, 8cm, 6cm, 8cm},width = .35\linewidth]{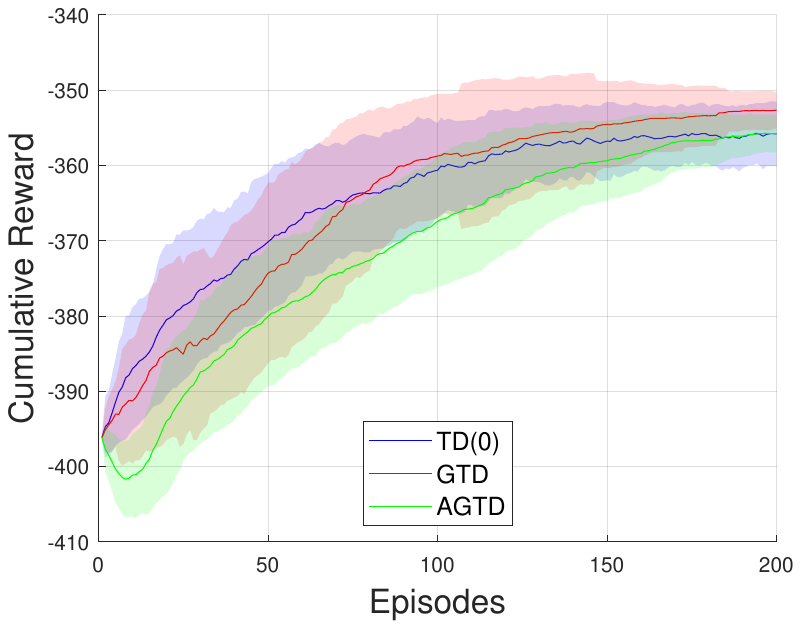} & \hspace{2cm}
		\includegraphics[trim = {6cm, 8cm, 6cm, 8cm},width = .35\linewidth]{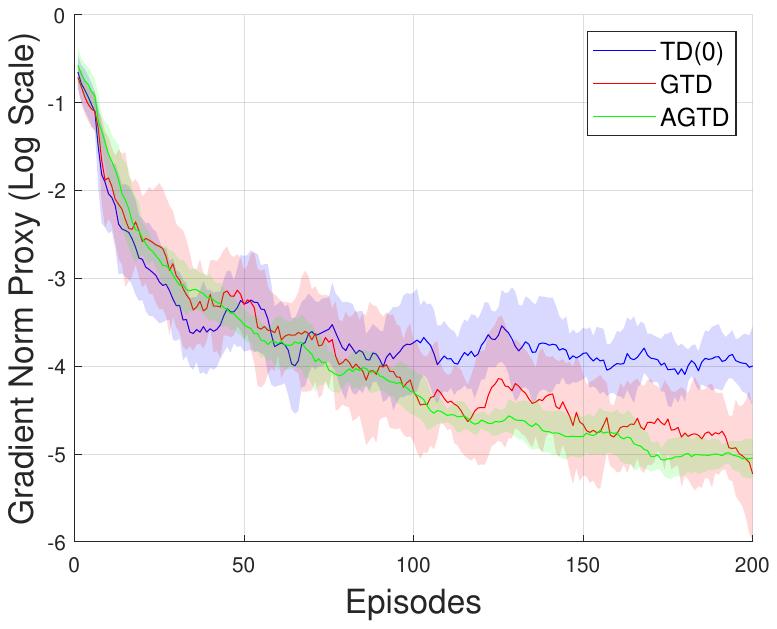} \\
		\small (a) & \hspace{2cm} \small (b)
		%
	\end{tabular}
	\caption{\blue{Pendulum Problem: (a) Average reward per episode with confidence bounds over 50 trials. (b) Average gradient norm proxy over 50 trials. In contrast to the navigation problem there is a significant gain in using advantage actor-critic; here, the state action ($Q$) function was used instead of the value function ($V$).} \newblue{A moving average filter of size ten has been applied on the gradient norm proxy to aid in comparison. }}
	\label{fig:pendulum_rates}
\end{figure}


%

\subsection{Discussion}
Recall that the analysis of Corollaries \ref{cor:GTD}, \ref{cor:GTD2}, and \ref{corr:cont} establish that the convergence rates for GTD, A-GTD, and TD(0) are $O(\epsilon^{-3})$, $O(\epsilon^{-5/2})$, and $O(\epsilon^{-2/\sigma})$ respectively [also see Table \ref{tab:rates}]. Figure \ref{fig:rates} shows the performance of the navigation problem with value and advantage function policy gradient updates. \newblue{As expected, A-GTD converges fastest with respect to the gradient norm proxy, while GTD and TD(0) are comparable. }
The plots highlight a disconnect between the convergence in reward and the convergence in gradient norm. Namely, TD converges faster in gradient norm, but slower with respect to the cumulative reward. Even more interesting, although AGTD converges fastest with respect to gradient norm \emph{and} reward, its resulting stationary point is suboptimal compared to TD and GTD (see Fig \ref{fig:gradients}). On the other hand, GTD and TD(0) converge the slower, and they consistently reach the \emph{solved} region \newblue{marked by the solid black line at $-66$}. We say that rewards which are greater than $-66$ are \emph{solved} trajectories because these trajectory spend time in the destination region. A trajectory which does not reach the destination region will have accumulated reward of $-66$ or less. Taken together, these theoretical and experimental results suggest a tight coupling between the choice of training methodology and the quality of learned policies. Thus, just as the choice of optimization method, statistical model, and sample size influence generalization in supervised learning, they do so in reinforcement learning. Theorem \ref{thm:general_rate} characterizes the rate of convergence to a stationary point of the Bellman optimality operator, however it does not provide any guarantee on the quality of the stationary point. Figure \ref{fig:gradients} captures this trade-off convergence rate and quality of the stationary point.  

The disconnect between convergence in reward and convergence in gradient norm appears again in the pendulum. \newblue{Figure \ref{fig:pendulum_rates} (b) shows the gradient norm proxy for the advantage actor-critic applied to the pendulum problem. Consistent with Table \ref{tab:rates}, AGTD converges the fastest with followed by GTD and TD(0). Here, we again see the disconnect between convergence in gradient norm and cumulative reward. Notice how in the first few iterations, TD(0) actually converges the fastest. In tandem, the cumulative reward of TD(0) also increases quickly. By the final episode, TD(0) and AGTD perform worse than GTD. This is consistent with the convergence rate and quality of stationary point trade-off observed in the navigation problem.} 

There are a number of future directions to take this work. To begin, we can establish bounds on cases where the samples are not i.i.d., but instead have Markovian noise. Second, we can further generalize our results to consider a generic critic convergence rate that does not necessarily take the form of \blue{Proposition \ref{prop:critic_bound}}. \blue{Third, we can explore the choice of feature representation to explicitly characterize the convergence rate of actor-critic with TD(0) critic updates with respect to $\lambda_\textrm{TD}$.} Finally, we can characterize the behavior of the variance and use such characterizations to accelerate training.


%
%
%
%

\bibliographystyle{spbasic}
\bibliography{RL3}
 
\section{Appendix}

\subsection{Proof of Lemma \ref{lem:submart_prop}} \label{prof:lem2}

%
%
By the Mean Value Theorem, there exists $\tilde \theta_k = \lambda\theta_k + (1-\lambda)\theta_{k+1}$ for some $\lambda \in [0,1]$ such that 
\begin{equation} \label{eqn:meanvaluetheorem}
J(\theta_{k+1}) = J(\theta_k) + (\theta_{k+1} - \theta_k)^\top \nabla J(\tilde \theta_k).
\end{equation}
%
%
Add and subtract $(\theta_{k+1}-\theta_k)^\top \nabla J(\theta_k)$ to the right hand side of \eqref{eqn:meanvaluetheorem} to obtain
\begin{equation} \label{equ:wkp1_eq}
J(\theta_{k+1}) =  J(\theta_k) + (\theta_{k+1} - \theta_k)^\top \left( \nabla J(\tilde \theta_k) -\nabla J(\theta_k) \right) + (\theta_{k+1}-\theta_k)^\top \nabla J(\theta_k).
 \end{equation}
By Cauchy Schwartz, we know $(\theta_{k+1} - \theta_k)^\top \left( \nabla J(\tilde \theta_k) - \nabla J(\theta_k)\right) \geq -\|\theta_{k+1} - \theta_k\|\| \nabla J(\tilde \theta_k) - J(\theta_k)\|$. Further, by the Lipschitz continuity of the gradient, we know $\|\nabla J(\tilde \theta_k) - \nabla J(\theta_k)\| \leq L \|\tilde \theta_k - \theta_k\|$. Therefore, we have 
\begin{equation}
(\theta_{k+1} - \theta_k)^\top \left( \nabla J(\tilde \theta_k) - J(\theta_k)\right) \geq -L \|\tilde \theta_k - \theta_k\| \cdot \|\theta_{k+1} - \theta_k\| \geq -L \| \theta_{k+1} - \theta_k\|^2 \; ,
\end{equation}
where the second inequality comes from substituting $\tilde \theta_{k} = (1-\lambda)\theta_{k+1} + \lambda \theta_k$. We substitute this expression into the definition of \bluest{$J(\theta_{k+1})$} in \eqref{equ:wkp1_eq} to obtain 
\begin{equation}
J(\theta_{k+1}) \geq  J(\theta_k) + (\theta_{k+1}-\theta_k)^\top \nabla J(\theta_k) -L\| \theta_{k+1} - \theta_k\|^2 .
 \end{equation}
Take the expectation with respect to the filtration $\F_k$, and substitute the definition for the actor update \eqref{eq:actor_update}
\begin{equation}
\E[J(\theta_{k+1})|\F_k] \geq  J(\theta_k) + \E [\theta_{k+1}-\theta_k | \F_k]^\top \nabla J(\theta_k) +  -L \E[\| \eta_k \hat g^{AC}_{H(k)}\|^2|\F_k] .
\end{equation}
We know from \eqref{eqn:actor_critic_bounded_variance} that $\|\hat \nabla J(\theta_k)\|\bluest{^2} \leq \sigma^2$, as such we obtain
%

\begin{equation}\label{eqn:before_lts2}
\E[J(\theta_{k+1}) | \F_k] \geq  J(\theta_k) + \E [\theta_{k+1}-\theta_k | \F_k]^\top \nabla J(\theta_k) -L \sigma^2\eta_k^2 .
\end{equation}
%
%
%
Therefore, we are left to show that the last term on the right-hand side of the preceding expression is ``nearly" an ascent direction. \blue{Recall from Algorithm \ref{alg:AC_generic} that the $k^\textrm{th}$ update takes the form \eqref{equ:fin_H_update}, that is
	\begin{equation}
	\E \left[\theta_{k+1}- \theta_k  | \F_k \right] = \eta_k  \E\left[\hat g^{AC}_{H(k)} | F_k\right] = \eta_k \E\left[\nabla_\theta J(\theta_k)| \F_k\right] + \eta_k\E\left[\hat g^{AC}_{H(k)} - \nabla_\theta J(\theta) | \F_k\right]
	\end{equation}
	Substituting into \eqref{eqn:before_lts2}, from Theorem \ref{thm:finite_bias}, we obtain 
	\begin{equation}
	\begin{split}
	\E[\bluest{J(\theta_{k+1})} | \F_k] &\geq \bluest{J(\theta_k)} +\eta_k \|\nabla_\theta J(\theta_k)\|^2 + \eta_k\E\left[\hat g^{AC}_{H(k)} - \nabla_\theta J(\theta_k) | \F_k\right] ^\top \nabla_\theta J(\theta_k)\bluest{-L\sigma^2\eta_k^2}\\
	& \geq \bluest{J(\theta_k)} +\eta_k \|\nabla_\theta J(\theta_k)\|^2 - \eta_k \left| \E\left[\hat g^{AC}_{H(k)} - \nabla_\theta J(\theta_k) | \F_k\right] ^\top \nabla_\theta J(\theta_k)  \right|\bluest{-L\sigma^2\eta_k^2}\\
	& \geq \bluest{J(\theta_k)} +\eta_k \|\nabla_\theta J(\theta_k)\|^2 - \eta_k \| \E\left[\hat g^{AC}_{H(k)}\right] - \nabla_\theta J(\theta_k)\| \cdot \|\nabla_\theta J(\theta_k)\|\bluest{-L\sigma^2\eta_k^2}\\
	&\geq \bluest{J(\theta_k)} +\eta_k \|\nabla_\theta J(\theta_k)\|^2 - \eta_k C_\nabla\left(C_1 \gamma^{H(k)-1} + C_2T(k)^{-b}\right)\bluest{-L\sigma^2\eta_k^2} 
	\end{split}
	\end{equation}
	This concludes the proof.
}

\subsection{Proof of Theorem \ref{thm:general_rate}} \label{proof:thm1}

\bluest{Take the total expectation of \eqref{equ:lem1_result} from Lemma \ref{lem:submart_prop}}
\begin{equation}\label{eq:theorem1_proof_expectation}
\E[J(\theta_{k+1}) ] \geq  \E[J(\theta_k)]  + \eta_k \E[ \| \nabla J(\theta_k)\|^2]  - \eta_k C_\nabla C_1 \gamma^{H(k)-1} - \eta_k C_\nabla C_2 T_C(k)^{-b}-\bluest{ L \sigma^2 \eta_k^2}.
\end{equation}
Define $U_k := J(\theta^*) - J(\theta_{k})$ where $\theta^*$ is the solution of \eqref{equ:max_goal} when the policy is parameterized by $\theta$. By this definition, we know that $U_k$ is non-negative for all $\theta_k$. Add $J(\theta^*)$ to both sides of the inequality and rearrange terms
\begin{equation}\label{eq:theorem1_proof_suboptimality}
\eta_k \E[\|\nabla J(\theta_k)\| ] \leq \E[U_k] - \E[U_{k+1}] + L \sigma^2 \eta_k^2 +  \eta_k C_\nabla C_1 \gamma^{H(k)-1} + \eta_k C_\nabla C_2 T_C(k)^{-b}.
\end{equation}
Divide both sides by $\eta_k$ and take the sum over $\{k - N, \dots, k\}$ for some integer $1 < N <k$
\begin{equation}\label{eq:theorem1_proof_sum}
\begin{split}
\sum_{j = k - N}^k \E[\|\nabla J(\theta_j)\|^2] \leq &\sum_{j = k-N}^k \frac{1}{\eta_j} \left( \E[U_j] - \E[U_{j +1}] \right) +  L \sigma^2\sum_{j = k-N}^k \eta_j \\
& + \sum_{j = k -N}^k  \left( C_\nabla C_1 \gamma^{H(j)-1}  C_\nabla C_2 T_C(j)^{-b}\right).
\end{split}
\end{equation}
Add and subtract $1/\eta_{k - N - 1} \E[U_{k -N}]$ on the right hand side. This allows us to write
\begin{align} 
\sum_{j = k - N}^k \E[\|\nabla J(\theta_j)\|^2] &\leq \sum_{j = k-N}^k \left( \frac{1}{\eta_j} -\frac{1}{\eta_{j -1}} \right)\E[U_{j}] - \frac{1}{\eta_k}\E[U_{k+1}] + \frac{1}{\eta_{k - N -1}}\E[U_{k - N}] \notag \\
 &+  L \sigma^2\sum_{j = k-N}^k \eta_j +  \sum_{j = k -N}^k  \left( C_\nabla C_1 \gamma^{H(j)-1}  C_\nabla C_2 T_C(j)^{-b}\right). \label{eq:theorem1_proof_decompose_sum}
\end{align}
By definition of $U_k$, $\E[U_{k+1}] \geq 0$. Therefore we can omit it from the right hand side of \eqref{eq:theorem1_proof_decompose_sum}. Further, we know that $J(\theta^*) \leq U_R/(1-\gamma)$ as a consequence from Assumption \ref{assum:regularity}\ref{as:bounded_reward} [see \eqref{equ:J_bound}]. Hence we have $U_k \leq 2U_R/(1-\gamma) =: C_3$ for all $k$. Substituting this fact into the preceding expression yields
\begin{equation}
\begin{split}
\sum_{j = k - N}^k \E[\|\nabla J(\theta_j)\|^2] \leq  & \sum_{j = k-N}^k \left( \frac{1}{\eta_j} -\frac{1}{\eta_{j -1}} \right)C_3 + \frac{1}{\eta_{k - N -1}}C_3  +  L \sigma^2\sum_{j = k-N}^k \eta_j \\
& +  \sum_{j = k -N}^k  \left( C_\nabla C_1 \gamma^{H(j)-1}  C_\nabla C_2 T_C(j)^{-b}\right).
\end{split}
\end{equation}
By unraveling the telescoping sum, the first two terms are equal to $C_3/\eta_k$
\begin{equation}
\sum_{j = k - N}^k \E[\|\nabla J(\theta_j)\|^2] \leq \frac{C_3}{\eta_k}  +  L \sigma^2\sum_{j = k-N}^k \eta_j + \sum_{j = k -N}^k  \left( C_\nabla C_1 \gamma^{H(j)-1}  C_\nabla C_2 T_C(j)^{-b}\right).
\end{equation}
Substitute $\eta_k = k^{-a}$ for the step size 
\begin{equation}\label{eq:thm1_proof_pre_cases}
\sum_{j = k - N}^k \E[\|\nabla J(\theta_j)\|^2] \leq  C_4k^a  +  L \sigma^2\sum_{j = k-N}^k j^{-a} + \sum_{j = k -N}^k  \left( C_\nabla C_1 \gamma^{H(j)-1}  C_\nabla C_2 T_C(j)^{-b}\right).
\end{equation}
We break the remainder of the proof into two cases due to the fact that the right-hand side of the preceding expression simplifies when $b=1$, and is more intricate when $0<b<1$. We focus on the later case first.

\vspace{2mm}\noindent {\bf Case (i): $b \in (0,1)$ } Consider the case where $b\in (0,1)$. Set $T_C(k) = k$ and $H(k) = k$. Substitute the integration rule, namely that $\sum_{j = k - N}^k j^{-a} \leq k^{1-a} -(k - N - 1)^{1-a}$, into \eqref{eq:thm1_proof_pre_cases} to obtain:
\begin{equation}\label{eq:eq:thm1_proof_case_1}
\begin{split}
\sum_{j = k - N}^k \E[\|\nabla J(\theta_j)\|^2] \leq & C_4k^a  + C_\nabla C_1 \gamma^{-1}\sum_{j = k-N}^k \gamma^{j} + \frac{L \sigma^2}{1- a} \left( k^{1 - a} - (k - N - 1)^{1 - a}\right) \\
& + \frac{CL_1}{1-b}\left( k^{1-b} -(k - N - 1)^{1-b}\right).
\end{split}
\end{equation}
Divide both sides by $k$ and set $N = k - 1$
\begin{equation}\label{eq:eq:thm1_proof_case_1_setN}
\begin{split}
\frac{1}{k}\sum_{j = 1}^k \E[\|\nabla J(\theta_j)\|^2] &\leq  C_4k^{a-1} + C_\nabla C_1 \gamma^{-1}k^{-1} \sum_{j = 1}^k \gamma^{j} + \frac{L \sigma^2}{1- a}   k^{- a}  + \frac{CL_1}{1-b}k^{-b}\\
&\leq C_4k^{a-1} + \frac{C_\nabla C_1}{ \gamma (1-\gamma)}k^{-1} + \frac{L \sigma^2}{1- a}   k^{- a}  + \frac{CL_1}{1-b}k^{-b}\\
\end{split}
\end{equation}
Suppose $k = K_\epsilon$ so that we may write
\begin{equation}\label{eq:thm1_proof_case_1_k_eps}
\frac{1}{K_\epsilon}\sum_{j = 1}^{K_\epsilon} \E[\|\nabla J(\theta_j)\|^2] \leq \mathcal{O}\left(K_\epsilon^{a-1} + K_\epsilon^{-1}+K_\epsilon^{-a}  + K_\epsilon^{-b}\right).
\end{equation}
By definition of $K_\epsilon$ [c.f. \eqref{equ:Keps_def}], we have that $\E[\|\nabla J(\theta_j)\|^2] > \epsilon$ for all $j = 1,\dots,K_\epsilon$, so
\begin{equation}
\epsilon \leq \frac{1}{K_\epsilon} \sum_{j = 1}^{K_\epsilon} \E[\| \nabla J(\theta_j)\|^2] \leq \mathcal{O}\left(K_\epsilon^{a-1}+K_\epsilon^{-1}+K_\epsilon^{-a}  + K_\epsilon^{-b}\right).
\end{equation}
Defining $\ell = \min\{a, 1-a, b\}$, the preceding expression then implies
\begin{equation}
\epsilon \leq \mathcal{O}(K_\epsilon^{-\ell}),
\end{equation}
which by inverting the expression, yields the sample complexity
\begin{equation}
K_\epsilon \leq \mathcal{O}(\epsilon^{-1/\ell}).
\end{equation}

\vspace{2mm}\noindent {\bf Case (ii): $b =1$ } Now consider the case where $b = 1$. Set $T_C(k) = k+1$ and $H(k) = k$. Again, using the integration rule, and that $\sum_{j = k - N}^k (j+1)^{-1} \leq \log(k+1) - \log(k - N )$, we substitute into \eqref{eq:thm1_proof_pre_cases} which yields
\begin{equation}\label{eq:thm1_proof_case_2}
\begin{split}
\sum_{j = k - N}^k \E[\|\nabla J(\theta_j)\|^2] \leq  & C_4k^a+C_\nabla C_1 \sum_{j = k - N}^{k}\gamma^{j}  +  \frac{L \sigma^2}{1- a} \left( k^{1 - a} - (k - N - 1)^{1 - a}\right) \\
& + CL_1\left(\log(k+1) -\log(k - N )\right).
\end{split}
\end{equation}
Divide both sides by $k$ and fix $N = k - 1$
\begin{equation}\label{eq:thm1_proof_case_2_fix_N}
\frac{1}{k}\sum_{j = 1}^k \E[\|\nabla J(\theta_j)\|^2] \leq  C_4k^{a-1} + C_\nabla C_1 \gamma^{-1}k^{-1} \sum_{j = 1}^k \gamma^{j}  + \frac{L \sigma^2}{1- a}   k^{- a}  + CL_1 \frac{\log(k+1)}{k}.
\end{equation}
Let $k = K_\epsilon$ in the preceding expression, which then becomes
\begin{equation}\label{eq:thm1_proof_case_2_k_eps}
\frac{1}{K_\epsilon}\sum_{j = 1}^{K_\epsilon} \E[\|\nabla J(\theta_j)\|^2] \leq \mathcal{O}\left(K_\epsilon^{a-1}+K_\epsilon^{-1}+K_\epsilon^{-a}  + \frac{\log( K_\epsilon+1)}{K_\epsilon}\right).
\end{equation}
Again, by definition of $K_\epsilon$ [c.f. \eqref{equ:Keps_def}], we have that $\E[\|\nabla J(\theta_j)\|^2] > \epsilon$ for all $j = 1,\dots,K_\epsilon$, so
\begin{equation}\label{eq:thm1_proof_case_1_eps}
\epsilon \leq \frac{1}{K_\epsilon} \sum_{j = 1}^{K_\epsilon} \E[\| \nabla J(\theta_j)\|^2] \leq \mathcal{O}\left(K_\epsilon^{a-1}+K_\epsilon^{-1}+K_\epsilon^{-a}  + \frac{\log( K_\epsilon+1)}{K_\epsilon}\right).
\end{equation}
Optimizing over $a$, we have
\begin{equation}
\epsilon \leq  \mathcal{O}\left( K_\epsilon^{-\frac{1}{2}}\right) \; \text{ for } b > \frac{1}{2}
\end{equation}
%
%
On the other hand, 
\begin{equation}
\epsilon \leq  \mathcal{O}\left( K_\epsilon^{-b}\right). \text{ for } b \leq 1/2
\end{equation}
%
%
Fix $\ell = \min\{1/2,b\}$, then 
\begin{equation}
\epsilon \leq \mathcal{O}(K_\epsilon^{-\ell}),
\end{equation}
which implies
\begin{equation}
K_\epsilon \leq \mathcal{O}(\epsilon^{-1/\ell}).
\end{equation}
This concludes the proof.

\end{document}